\documentclass[11pt]{article} 
\usepackage[accepted]{tmlr}
\usepackage[utf8]{inputenc}
\usepackage{microtype}
\usepackage{amsmath,amssymb}
\usepackage{amsthm}
\usepackage{mathtools}
\usepackage{mathrsfs}
\usepackage{xcolor}
\usepackage{tikz}
\usepackage{graphicx}
\usepackage{algorithm}
\usepackage{algpseudocode}
\usepackage{float}
\usepackage{enumerate}
\usepackage{subcaption}
\newtheorem{theorem}{Theorem}[section]
\newtheorem{definition}[theorem]{Definition}
\newtheorem{proposition}[theorem]{Proposition}

\newtheorem{example}[theorem]{Example}
\newtheorem{remark}[theorem]{Remark}
\usepackage{enumitem}

\usepackage{pgf}
\usepackage{tikz}
  \usetikzlibrary{arrows,automata}
  \usetikzlibrary{spy}
\usepackage{pgfplots}
  \pgfplotsset{width=7cm,compat=1.8}

\usepackage{mwe}
\usepackage{ragged2e}

\newcommand{\R}{\mathbb{R}}
\DeclareMathOperator*{\argmin}{arg\,min}

\DeclareMathOperator*{\sign}{sign}
\DeclareMathOperator*{\zer}{zer}
\DeclareMathOperator*{\dom}{dom}
\DeclareMathOperator{\prox}{prox}

\DeclareMathOperator{\id}{id}
\DeclareMathOperator{\Fix}{Fix}
\DeclareMathOperator{\PnP}{PnP}
\DeclareMathOperator{\range}{range}
\DeclareMathOperator{\diag}{diag}
\DeclareMathOperator{\Id}{id}

\newcommand{\RegFunc}{g}
\newcommand{\NNparam}{\theta}
\newcommand{\ForwardOp}{A}
\newcommand{\Expect}{\mathbb{E}}
\newcommand{\Real}{\mathbb{R}}

\newcommand{\Op}[1]{\operatorname{\mathcal{#1}}}

\newcommand{\stx}{\mathbf{x}}
\newcommand{\sty}{\mathbf{y}}
\newcommand{\stw}{\mathbf{w}}

\makeatletter
\def\plist@algorithm{Alg.\space}

\makeatother

\usepackage[colorlinks=true, citecolor=blue]{hyperref}
\usepackage{cleveref}
\usepackage{url}

\title{Unsupervised approaches based on optimal transport and convex analysis for inverse problems in imaging}

\author{\name Marcello Carioni \email m.c.carioni@utwente.nl\\
      \addr Department of Applied Mathematics\\ 
      University of Twente, Enschede, Netherlands
      \AND
      \name Subhadip Mukherjee \email smukherjee@ece.iitkgp.ac.in\\
      \addr Department of Electronics and Electrical Communication Engineering\\
      Indian Institute of Technology (IIT), Kharagpur, India
      \AND
      \name Hong Ye Tan \email hyt35@cam.ac.uk \\
      \addr Department of Applied Mathematics and Theoretical Physics \\
      University of Cambridge, United Kingdom
      \AND
      \name Junqi Tang \email j.tang.2@bham.ac.uk\\
      \addr School of Mathematics \\
      University of Birmingham, United Kingdom
      }

\begin{document}

\maketitle

\begin{abstract}
Unsupervised deep learning approaches have recently become one of the crucial research areas in imaging owing to their ability to learn expressive and powerful reconstruction operators even when paired high-quality training data is scarcely available. In this chapter, we review theoretically principled unsupervised learning schemes for solving imaging inverse problems, with a particular focus on methods rooted in optimal transport and convex analysis. We begin by reviewing the optimal transport-based unsupervised approaches such as the cycle-consistency-based models and learned adversarial regularization methods, which have clear probabilistic interpretations. Subsequently, we give an overview of a recent line of works on provably convergent learned optimization algorithms applied to accelerate the solution of imaging inverse problems, alongside their dedicated unsupervised training schemes. We also survey a number of provably convergent plug-and-play algorithms (based on gradient-step deep denoisers), which are among the most important and widely applied unsupervised approaches for imaging problems. At the end of this survey, we provide an overview of a few related unsupervised learning frameworks that complement our focused schemes. Together with a detailed survey, we provide an overview of the key mathematical results that underlie the methods reviewed in the chapter to keep our discussion self-contained.
\end{abstract}

\section{Introduction} 
Inverse problems seek to estimate an unknown parameter $x^*\in \mathbb{X}$ from a degraded measurement of the form 
\begin{equation}
    y=Ax^*+\mathrm{w} \in \mathbb{Y},
    \label{eq:measurement_eq}
\end{equation}
where $\mathrm{w}$ represents measurement error (noise) and $A:\mathbb{X}\rightarrow \mathbb{Y}$ is an operator that encodes the physical phenomenon governing the data acquisition process. Throughout this chapter, we will consider linear inverse problems, where $A$ is a bounded linear operator between two normed vector spaces $\mathbb{X}$ and $\mathbb{Y}$. Inverse problems are ubiquitous in imaging applications, especially in medical image reconstruction. A classic example of an inverse problem is image recovery in X-ray computed tomography (CT), where $\mathbb{X}$ is a set of functions on $\mathbb{R}^3$ (or on a subset of $\mathbb{R}^3$). The tomographic measurement data in the absence of noise is given by line integrals of the form
\begin{equation}
y(\ell)=\int_{\ell}x^*(z)\,\mathrm{d}z,\text{\,\,where\,\,}\ell\in\mathcal{L}.
    \label{eq:measurement_eq_ct}
\end{equation}
Here, $\mathcal{L}$ represents a pre-specified set of lines in $\mathbb{R}^3$. In other words, the measurement in X-ray CT consists of projections along a set of lines determined by the acquisition geometry, and the corresponding inverse problem seeks to recover the underlying true image $x^*$. Other notable examples of imaging inverse problems include magnetic resonance imaging (MRI), super-resolution microscopy, inpainting, image deblurring, compressed sensing, etc.

Without any further information about $x^*$, inverse problems are generally ill-posed, meaning that, there could be either no solutions or several possible solutions $x$ satisfying the operator equation \eqref{eq:measurement_eq}, even without any measurement noise. In the classical function-analytic setting, the underlying image $x^*$ is assumed to be an unknown deterministic parameter, and the noise $\mathrm{w}$ is assumed to be bounded, i.e., $\|\mathrm{w}\|_{\mathbb{Y}}\leq \delta$ for some $\delta>0$. The task is to then construct a family of reconstruction operators $G^{\lambda}:\mathbb{Y}\rightarrow \mathbb{X}$, parameterized by $\lambda$, such that $G^{\lambda}(y)$ yields a reasonable approximation of $x^*$. Variational regularization has by far been the most popular approach to construct such reconstruction maps by defining them as a minimizer of a variational energy function:
\begin{equation}
    G^{\lambda}(y)\in\underset{x\in \mathbb{X}}{\argmin}\,f(Ax,y)+ R_{\lambda}(x).
    \label{eq:var_reg}
\end{equation}
Here, $f:\mathbb{Y}\times \mathbb{Y}\rightarrow \mathbb{R}^+$ measures data fidelity and $R_{\lambda}:\mathbb{X}\rightarrow \mathbb{R}$ is a regularization function (\textit{regularizer} in short), parameterized by $\lambda$, that encodes prior knowledge about the reconstructed image. A popular choice is to construct the regularizer as $R_{\lambda}(x)=\lambda\,R(x)$, where $R$ is a fixed regularizer and $\lambda\in\mathbb{R}^+$ is a penalty parameter balancing data fidelity and regularization. The classical regularization theory for inverse problems deals with the construction of regularizers $R_{\lambda}$ such that $G^{\lambda}(y)$ varies continuously in $y$ (\textit{stability}), and that there exists a parameter selection rule $\lambda:\delta\mapsto \lambda(\delta)$ such that as the noise level $\delta\rightarrow 0$, $G^{\lambda(\delta)}(y)$ converges to a generalized solution of the noiseless operator equation $y^0=Ax$, where $y^0$ denotes the noise-free measurement. Such a family of reconstruction operators $(G^{\lambda(\delta)})_{\delta> 0}$ is said to be a \textit{convergent regularization scheme} \cite{benning2018modern}.    

An alternative modeling approach for inverse problems is offered by the Bayesian framework, wherein a possible image $x$ and its measurement $y$ are treated as realizations of the $\mathbb{X}$- and $\mathbb{Y}$-valued random variables $\mathbf{x}$ and $\mathbf{y}$, respectively. The goal of Bayesian inversion is to characterize the full posterior distribution $p_{\text{post}}$ of $\mathbf{x}$ conditioned on $\mathbf{y}$ by utilizing Bayes' formula
\begin{equation*}
    p_{\text{post}}(x|y) = \frac{1}{Z(y)}\,p_{\mathrm{w}}(y-Ax)\,p_0(x),
    \label{eq:bayes}
\end{equation*}
where $p_0$ is the \textit{prior} probability density on $\mathbf{x}$ and $Z(y)$ is a normalizing constant. The data likelihood is specified through the distribution $p_{\mathrm{w}}$ of the noise and the forward operator $A$. If $p_0$ is a Gibbs prior of the form $p_0(x)\propto \exp(-R_{\lambda}(x))$, the maximum a-posteriori probability (MAP) estimate of $x$ leads to a variational optimization of the form \eqref{eq:var_reg} akin to the function-analytic setting. It is worth emphasizing that the Bayesian approach can, in principle, go beyond \textit{point estimation} and allow for \textit{uncertainty quantification} in the solution using the complete posterior distribution (albeit with higher computational complexity). In the context of Bayesian inversion, the notion of \textit{stability} refers to the continuity (with respect to $y$) of the posterior, while convergence in the Bayesian framework deals with the concentration of the posterior measured in a suitable metric \cite{dashti2017bayesian}. 

In the classical model-driven variational approach, the regularizer (or, equivalently, the prior in the Bayesian setting) is constructed analytically to promote certain smoothness properties in the underlying image. Some notable choices for the regularizer include \textit{Tikhonov} regularization ($R_{\lambda}(x)=\lambda\,\|Bx\|_2^2$, where $B$ is a bounded linear functional), the \textit{total variation} (TV) regularizer ($R_{\lambda}(x)=\lambda\,\|\nabla x\|_1$), and more recently, sparsity-promoting regularizers (seeking to encourage the image to be sparse in a fixed or learned basis) \cite{scherzer2009variational}. While model-driven approaches for inverse problems have been studied extensively over the past few decades, the success of deep learning has led to the emergence of data-driven methods for solving imaging inverse problems in recent years. These methods not only surpass the classical model-driven approaches in terms of empirical performance, but some of the data-driven methods also come with theoretical guarantees (see \cite{ieee_spm_guarantees} and references therein). The data-driven methods can broadly be classified into two categories, namely, \textit{supervised} and \textit{unsupervised}. Roughly speaking, supervised methods work in an end-to-end manner and need access to the ground-truth images to be compared against the output of a learned reconstruction operator, as opposed to unsupervised methods which do not rely on the availability of such ground-truths for a direct point-by-point comparison with the learned reconstruction. Therefore, unsupervised methods offer greater flexibility over supervised approaches in terms of the required training dataset for learning the parameters of the reconstruction operator, thereby leading to better practical usability. 

The objective of this chapter is to provide a survey of learned unsupervised methods for inverse problems, focusing particularly on approaches that leverage ideas from generative machine learning (and optimal transport, in particular) and classical (convex) optimization theory. We first provide an extensive mathematical background on optimal transport and convex analysis highlighting the important concepts that underlie such approaches, followed by a detailed review of the notable unsupervised approaches in the context of imaging inverse problems. The survey aims to highlight the key mathematical foundations behind the development of unsupervised learning approaches and underscores the potential of unsupervised methods in achieving competitive empirical performance as compared to their supervised counterparts.

\noindent \textbf{Outline of the chapter.} The chapter is organized as follows. Section \ref{sec:bg} provides the necessary mathematical background that will be used throughout the chapter, from optimal transport to convex analysis. In this section, we also describe classical methods for learning reconstruction operators for imaging inverse problems, while also highlighting the key differences between supervised and unsupervised approaches. \Cref{sec:OTapproaches} presents recent unsupervised approaches to inverse problems based on optimal transport, focusing particularly on cycle-consistency-based models and learned adversarial regularizers. \Cref{sec:unsupervised} surveys various unsupervised regularization-based approaches to inverse problems, with methods ranging from learning optimizers for model-based reconstruction to Plug-and-Play (PnP) denoising methods based on implicitly defined denoising priors. Finally, in \Cref{ssec:GTfreeIR}, we review some notable ground-truth-free approaches for image reconstruction approaches that have been shown to result in impressive empirical performance in numerous practical applications. 
\section{Background}\label{sec:bg} 
In this section, we provide the necessary mathematical background needed in the remainder of the chapter to make the exposition self-contained as much as possible. In particular, we provide a detailed overview of some of the important results in optimal transport and convex analysis, which serve as the conceptual foundation of the unsupervised techniques that we subsequently review in this chapter. Further, we also precisely characterize what we mean by \textit{supervised} and \textit{unsupervised} learning approaches, considering the vagueness around how these terms can possibly be interpreted.
\subsection{Probability measures}
Since the data-driven methods reviewed in this article depend heavily on approximating unknown probability measures, we give a formal overview of the key definitions and results in probability theory that will be useful for us. In particular, we formally define some key concepts related to probability measures and random variables, followed by a short description of different notions of distance between two probability measures.  
\subsubsection{Probability space and random variables}
A probability space consists of the triplet $\left(\Omega,\mathcal{F},\pi\right)$, where $\Omega$ is the sample space, $\mathcal{F}$ is a $\sigma$-algebra consisting of subsets of $\Omega$, and $\pi:\mathcal{F}\rightarrow [0,1]$ is a probability measure. We will assume $\Omega$ to be a \textit{Polish space} (i.e., a complete metric space with an underlying metric $d:\Omega\times \Omega\rightarrow [0,+\infty]$ and a countable dense subset). We will use the notation $\mathcal{P}(\Omega)$ to denote the set of all possible probability measures on the measurable space $\left(\Omega,\mathcal{F}\right)$. 

\noindent For any $\mathbb{R}^d$-valued random variable $\mathbf{x}$ on a probability space $\left(\Omega,\mathcal{F},\pi\right)$, the corresponding probability law is defined as the following probability measure $\pi_{\mathbf{x}}$ on $(\mathbb{R}^d,\mathcal{B}(\mathbb{R}^d))$, where $\mathcal{B}(\mathbb{R}^d))$ denotes the Borel $\sigma$-algebra of $\mathbb{R}^d$:
\begin{equation*}
    \pi_{\mathbf{x}}(\mathscr{B}) := \pi(\mathbf{x}^{-1}(\mathscr{B}))\text{\,\,for all\,\,}\mathscr{B}\in\mathcal{B}(\mathbb{R}^d).
\end{equation*}
Let $\lambda$ be the Lebesgue measure on $(\mathbb{R}^d,\mathcal{B}(\mathbb{R}^d))$. If there exists a nonnegative function $p_{\mathbf{x}}:\mathbb{R}^d \rightarrow [0,+\infty]$ such that $\pi_{\mathbf{x}}(\mathscr{B})=\int_{\mathscr{B}}p_{\mathbf{x}}\,\mathrm{d}\lambda$ for all $\mathscr{B}\in\mathcal{B}(\mathbb{R}^d)$, $p_{\mathbf{x}}$ is called the \textit{probability density function} (p.d.f.), or simply the \textit{density} of $\mathbf{x}$ (or $\pi_{\mathbf{x}}$). The existence of $p_{\mathbf{x}}$ is guaranteed by the Radon-Nikodym theorem if $\pi_{\mathbf{x}}$ is \textit{absolutely continuous} with respect to $\lambda$, i.e., if $\pi_{\mathbf{x}}(\mathscr{B})=0$ whenever $\lambda(\mathscr{B})=0$, for any $\mathscr{B}\in\mathcal{B}(\mathbb{R}^d)$. In this case, the density $p_{\mathbf{x}}$ is usually written as $p_{\mathbf{x}}=\frac{\mathrm{d}\pi_{\mathbf{x}}}{\mathrm{d}\lambda}$, the Radon-Nikodym derivative of $\pi_{\mathbf{x}}$ with respect to the Lebesgue measure $\lambda$. The density $p_{\mathbf{x}}$, if exists, is unique $\lambda$-almost everywhere (a.e.). 

\noindent The expected value of $\mathbf{x}$, denoted as $\mathbb{E}_{\mathbf{x}\sim\pi_{\mathbf{x}}}[\mathbf{x}]$ or simply $\mathbb{E}[\mathbf{x}]$, is defined as 
\begin{equation}
    \mathbb{E}_{\mathbf{x}\sim\pi_{\mathbf{x}}}[\mathbf{x}]:=\int_{\Omega}X(\omega)\,\mathrm{d}\pi(\omega)=\int_{\mathbb{R}^d}x\,\mathrm{d}\pi_{\mathbf{x}}(x).
    \label{eq:def_expectation}
\end{equation}
If $\mathbf{x}$ has a density $p_{\mathbf{x}}$, the expectation defined in \eqref{eq:def_expectation} can equivalently be written as 
\begin{equation}
    \mathbb{E}_{\mathbf{x}\sim\pi_{\mathbf{x}}}[\mathbf{x}]:=\int_{\mathbb{R}^d}x\,p_{\mathbf{x}}\,(x)\mathrm{d}\lambda(x).
    \label{eq:def_expectation_pdf}
\end{equation}
Any mapping $T:\left(\Omega_1,\mathcal{F}_1\right)\rightarrow \left(\Omega_2,\mathcal{F}_2\right)$ between two measurable spaces with the property that 
\begin{equation}
    T^{-1}(\mathscr{A}) \in \mathcal{F}_1 \text{\,\,for all\,\,}\mathscr{A}\in \mathcal{F}_2,
    \label{eq:measurability}
\end{equation}
is said to be a \textit{measurable function} from $\left(\Omega_1,\mathcal{F}_1\right)$ to $\left(\Omega_2,\mathcal{F}_2\right)$.  Let $\pi_1$ be a probability measure on $\left(\Omega_1,\mathcal{F}_1\right)$. The push-forward measure of $\pi_1$ by $T$, denoted as $T_{\#}\pi_1$, is defined as a probability measure on $\left(\Omega_2,\mathcal{F}_2\right)$ such that 
\begin{equation*}
    T_{\#}\pi_1(\mathscr{A}) = \pi_1(T^{-1}(\mathscr{A})) \text{\,\,for all\,\,}\mathscr{A}\in \mathcal{F}_2.
\end{equation*}
For a measurable function $g:(\mathbb{R}^d,\mathcal{B}(\mathbb{R}^d))\rightarrow(\mathbb{R}^m,\mathcal{B}(\mathbb{R}^m))$, the expected value of $\mathbf{y}=g(\mathbf{x})$ is defined as the following integral:
\begin{equation}
    \mathbb{E}_{\mathbf{x}\sim\pi_{\mathbf{x}}}[g(\mathbf{x})]:=\int_{\mathbb{R}^d}g(x)\,\mathrm{d}\pi_{\mathbf{x}}(x)=\int_{\mathbb{R}^m}y\,\mathrm{d}\pi_{\mathbf{y}}(y), \text{\,\,where\,\,}\pi_{\mathbf{y}}:=g_{\#}\pi_{\mathbf{x}}.
    \label{eq:def_expectation_func}
\end{equation}

\subsubsection{Distance between probability measures}\label{sec:notionsdistances}
Many data-driven approaches for inverse problems rely on methods that are able to efficiently estimate and minimize the distance between two probability distributions. A notable class of distances between probability measures is given by the class of $\phi$-divergences (sometimes referred to as $f$-divergences), containing common metrics such as the Kullback-Leibler divergence and the total variation distance.
\begin{definition}[$\phi$-divergence]\label{def:phi_div}
Let $\pi_\mathbf{x}$ and $\pi_\mathbf{y}$ be two probability measures on $(\Omega,\mathcal{F})$ with $\pi_\mathbf{x}$ being absolutely continuous with respect to $\pi_\mathbf{y}$. Let $\phi:(0,+\infty)\rightarrow (-\infty,+\infty)$ be a convex function such that $\phi(1)=0$, and $\phi(0)\coloneqq\underset{t\rightarrow 0^{+}}{\lim}\,\phi(t)$ (which could be infinite). Then, the $\phi$-divergence between $\pi_\mathbf{x}$ and $\pi_\mathbf{y}$ is defined as
\begin{equation}
        D_{\phi}(\pi_\mathbf{x},\pi_\mathbf{y}):=\int_{\Omega}\phi\left(\frac{\mathrm{d}\pi_\mathbf{x}}{\mathrm{d}\pi_\mathbf{y}}(\omega)\right)\,\mathrm{d}\pi_{\mathbf{y}}(\omega),
        \label{eq:phi_div_def}
\end{equation}
where $\frac{\mathrm{d}\pi_\mathbf{x}}{\mathrm{d}\pi_\mathbf{y}}$ is the Radon-Nikodym derivative of $\pi_\mathbf{x}$ with respect to $\pi_{\mathbf{y}}$.
\end{definition} 
Consider now the special case (albeit an important one) where $(\Omega,\mathcal{F})=(\mathbb{R}^d,\mathcal{B}(\mathbb{R}^d))$ and let $p_{\mathbf{x}}$ and $p_{\mathbf{y}}$ be the densities of $\pi_\mathbf{x}$ and $\pi_\mathbf{y}$, respectively. Then, \eqref{eq:phi_div_def} can be rewritten as
\begin{equation}
        D_{\phi}(p_\mathbf{x},p_\mathbf{y}):=\int_{\Omega}\phi\left(\frac{p_{\mathbf{x}}(x)}{p_{\mathbf{y}}(x)}\right)\,p_{\mathbf{y}}(x)\,\mathrm{d}\lambda(x).
        \label{eq:phi_div_def_density}
\end{equation}
The following are some important special instances of $\phi$-divergence that are useful in the context of machine learning and inverse problems:
\begin{enumerate}
    \item Kullback–Leibler (KL): $\phi(t)=t\log t$.
    \item Jensen-Shannon (JS): $\phi(t)=-(t+1)\log\left(\frac{t+1}{2}\right)+t\log t$. 
    \item Total Variation (TV): $\phi(t)=\frac{1}{2}|t-1|$.
    \item Squared Hellinger: $\phi(t)=(\sqrt{t}-1)^2$.
\end{enumerate}
Despite their popularity, $\phi$-divergences have been shown to have serious practical limitations, especially in the context of machine learning, imaging, and inverse problems. From a theoretical point of view, this is due to the fact that they are well-defined only if $\pi_{\mathbf{x}}$ is absolutely continuous with respect to $\pi_{\mathbf{y}}$. Consequently, $\phi$-divergences are not well-suited to compare probability distributions concentrated on low dimensional manifolds. This problem often arises in the context of distribution learning for imaging problems, where data-sets can reasonably be approximated as low dimensional manifolds embedded in a high-dimensional ambient space.  

Another popular family of distance measures that overcome some of the shortcomings of $\phi$-divergences are the so-called integral probability metrics. They have the advantage that they are also well-defined for singular measures and are easier to estimate from finitely many samples as compared with $\phi$-divergences, especially in high-dimensional settings \cite{sriperumbudur2009integral}.
\begin{definition}[Integral probability metrics (IPMs)]\label{def:IPM}
Let $\pi_\mathbf{x}$ and $\pi_\mathbf{y}$ be two probability measures on $(\Omega,\mathcal{F})$, and let $\mathcal{G}$ be some class of bounded and measurable functions $g:\Omega\rightarrow\mathbb{R}$. Integral probability metrics are defined as
\begin{equation}
    \Delta_{\mathcal{G}}(\pi_\mathbf{x},\pi_\mathbf{y}):=\underset{g\in\mathcal{G}}{\sup}\left|\int_{\Omega}g(\omega)\,\mathrm{d}\pi_\mathbf{x}(\omega)-\int_{\Omega}g(\omega)\,\mathrm{d}\pi_\mathbf{y}(\omega)\right|.
    \label{eq:def_IPM}
\end{equation}
\end{definition}

Relevant examples of integral probability metrics are the following:

\medskip

\begin{enumerate}
\item Total-variation distance (TVD): $\mathcal{G} = \{g \in C(\Omega) : \sup_{x\in \Omega} g(x) \leq 1\}$. Note that in the case where $\pi_{\mathbf{x}}$ and $\pi_{\mathbf{y}}$ have densities with respect to the Lebesgue measure, this definition of the TVD is equivalent to TV defined as a $\phi$-divergence corresponding to $\phi(t)=\frac{1}{2}|t-1|$.
\item Maximum-mean-discrepancy (MMD): $\mathcal{G} = \{g \in \mathcal{H}: \|g\|_{\mathcal{H}} \leq 1\}$ where $\mathcal{H}$ is a Reproducing Kernel Hilbert Space (RKHS).
\item  Kolmogorov distance (KD):  $\mathcal{G} = \{1_{(-\infty, t)} : t \in \mathbb{R}\}$.
\item $1$-Wasserstein distance: $\mathcal{G}$ is the class of 1-Lipschitz functions.
\end{enumerate}
We refer the interested reader to \cite{sriperumbudur2009integral} for more details.
It is worth mentioning that not all IPMs are suitable for comparing distributions, and the choice of $\mathcal{G}$ should be made depending on the problem under consideration. For instance, given $s \in \R$, suppose that $\pi_0 \in \mathcal{P}(\R^2)$ is concentrated and uniformly distributed on the segment $[0,1] \times \{0\} \subset \R^2$ and $\pi_s \in \mathcal{P}(\R^2)$ is concentrated and uniformly distributed on the segment $[0,1] \times \{s\} \subset \R^2$. Then, it is easy to verify that
\begin{align*}
{\rm TVD}(\pi_0,\pi_s) = \begin{cases}
    2, & \text{for $s\neq 0$},\\
    0, & \text{for $s=0$}.
\end{cases}
\end{align*}
This shows that the total variation is agnostic to the relative positions of the segments in the plane and it is thus not a suitable metric to compare the two distributions. This observation easily translates to any two distributions concentrated on disjoint lower dimensional manifolds. Moreover, it prevents the use of gradient descent strategies due to the severity of the vanishing gradient phenomenon.
We will see in Section \ref{sec:wasserstein} that such shortcomings are alleviated by the $1$-Wasserstein distance, making this a popular choice for machine learning applications pertaining to image processing.

\subsection{Optimal transport}
In this section, we recall some fundamental definitions and results in optimal transport that are relevant to the development of the unsupervised learning approaches discussed in \Cref{sec:OTapproaches}.
\subsubsection{Monge and Kantorovich formulations of optimal transport}
\noindent Let $\Omega_1$ and $\Omega_2$ be two Polish spaces. Correspondingly, consider two Borel probability spaces $\left(\Omega_1,\mathcal{B}(\Omega_1),\pi_1\right)$ and $\left(\Omega_2,\mathcal{B}(\Omega_2),\pi_2\right)$, where $\mathcal{B}(\Omega_1)$ and $\mathcal{B}(\Omega_2)$ are the Borel $\sigma$-algebras of $\Omega_1$ and $\Omega_2$, respectively. Let $c:\Omega_1\times\Omega_2\rightarrow[0,+\infty]$ be the cost of transporting one unit of mass from $x\in \Omega_1$ to $y\in\Omega_2$. We will assume that the cost $c$ is continuous. Monge's optimal transport problem \cite{monge1781memoire} is formulated as 
\begin{equation}
   M_c(\pi_1, \pi_2) =  \underset{T:T_{\#}\pi_1=\pi_2}{\inf}\int_{\Omega_1}c(x,T(x))\,\mathrm{d}\pi_1(x),
    \label{eq:monge_ot}
\end{equation}
where the minimization is carried out over measurable maps $T$. In other words, \eqref{eq:monge_ot} seeks to find a transport map $T$ (i.e., a mapping $T$ satisfying $T_{\#}\pi_1=\pi_2$) for which the overall transportation cost is minimized. Monge's problem may not have a solution, and even worse, a transport map may not always exist \cite{santambrogio2015optimal} (consider, for instance, the problem of transporting discrete masses from one set of locations to another). Due to this shortcoming, Kantorovich proposed a relaxation of \eqref{eq:monge_ot} in order to restore the well-posedness of the variational problem \cite{kantorovich2006translocation}.
The Kantorovich relaxation of \eqref{eq:monge_ot} reformulates the optimal transport problem by instead considering the transportation of mass from any $x\in \Omega_1$ to any $y\in \Omega_2$. Let $\pi$ be a probability measure on the product space $(\Omega_1\times \Omega_2,\mathcal{B}(\Omega_1) \otimes \mathcal{B}(\Omega_2))$, where $\mathcal{B}(\Omega_1) \otimes \mathcal{B}(\Omega_2)$ is the smallest $\sigma$-algebra generated by $\mathcal{B}(\Omega_1) \times \mathcal{B}(\Omega_2)$. We call $\pi$ a \textit{transport plan} if   
\begin{equation*}
    \pi(\mathscr{A}\times \Omega_2) = \pi_1(\mathscr{A}) \ \text{\,\,and\,\,}\ \pi(\Omega_1 \times \mathscr{B}) = \pi_2(\mathscr{B})\ \text{\,\,for all\,\,}\ \mathscr{A}\in\mathcal{B}(\Omega_1), \mathscr{B}\in\mathcal{B}(\Omega_2). 
\end{equation*}
Equivalently, $\pi$ is a transport plan if its marginals are $\pi_1$ and $\pi_2$. Let $\Pi(\pi_1,\pi_2)$ be the collection of all transport plans from $\left(\Omega_1,\mathcal{B}(\Omega_1),\pi_1\right)$ to $\left(\Omega_2,\mathcal{B}(\Omega_2),\pi_2\right)$. The Kantorovich relaxation of Monge's problem \eqref{eq:monge_ot} is then given by
\begin{equation}\label{eq:kantorovich_ot}
K_c(\pi_1, \pi_2) =    \underset{\pi:\pi\in \Pi(\pi_1,\pi_2)}{\min}\int_{\Omega_1\times\Omega_2}c(x,y)\,\mathrm{d}\pi(x,y).
\end{equation}
Note that the set $\Pi(\pi_1,\pi_2)$ is non-empty, since the product measure $\pi_1 \otimes \pi_2 \in \Pi(\pi_1,\pi_2)$. Moreover, the existence of a transport plan minimizing \eqref{eq:kantorovich_ot} is a consequence of Prokhorov's theorem and the narrow continuity of the map $\pi \mapsto \int_{\Omega_1\times\Omega_2}c(x,y)\,\mathrm{d}\pi(x,y)$ in $\mathcal{P}(\Omega_1 \times \Omega_2)$ \cite{santambrogio2015optimal}. Notably, \eqref{eq:kantorovich_ot} is a relaxation of Monge's problem since given a transport map $T$, one can construct the associated transport plan as $\pi = (\Id \times T)_\# \pi_1$.
Under certain conditions, the Kantorovich formulation \eqref{eq:kantorovich_ot} can be shown to be equivalent to the Monge formulation \eqref{eq:monge_ot} such as in the following theorem \cite[Theorem B]{pratelli2007equality}.

\begin{theorem}
If $\pi_1$ is non-atomic, namely $\pi_1(\{x\}) = 0$ for all $x \in \Omega_1$, then 
\begin{align*}
    M_c(\pi_1,\pi_2) = K_c(\pi_1, \pi_2).
\end{align*}
\end{theorem}

Standard arguments of convex analysis also ensure the equivalence of \eqref{eq:monge_ot} and \eqref{eq:kantorovich_ot} in the case where $\pi_1$ and $\pi_2$ are empirical measures with uniform weights, i.e.,
\begin{align}
    \pi_1 = \frac{1}{N}\sum_{i=1}^N \delta_{x_i}, \quad \pi_2 = \frac{1}{N}\sum_{i=1}^N \delta_{y_i},
\end{align}
where $\delta_{u}$ denotes the Dirac measure at $u$.
\subsubsection{The Wasserstein distance}\label{sec:wasserstein}
The Kantorovich formulation of optimal transport allows us to define a distance between two probability measures in a way that overcomes the shortcomings of the distances defined in Section \ref{sec:notionsdistances}. Given a distance $d : \Omega \times \Omega \rightarrow [0,\infty)$, for $1 \leq p <\infty$ the \emph{$p$-Wasserstein distance} between two Borel probability measures $\pi_1, \pi_2 \in \mathcal{P}(\Omega)$ is defined in terms of the Kantorovich formulation \eqref{eq:kantorovich_ot} as
\begin{align}
    W_p(\pi_1, \pi_2) \coloneqq (K_{d^p}(\pi_1, \pi_2))^{1/p}, \quad \pi_1, \pi_2 \in \mathcal{P}(\Omega).
\end{align}
It can be shown that $W_p$ defines a distance metric on the space of probability measures $\mathcal{P}(\Omega)$ \cite{santambrogio2015optimal}.
Moreover, it addresses some of the issues of $\phi$-divergences and IPMs presented in Section \ref{sec:notionsdistances}. First, it is well-defined for any pair of probability measures $\pi_1,\pi_2$, even if they are mutually singular. Moreover, following the example given in Section \ref{sec:notionsdistances}, let $\pi_0 \in \mathcal{P}(\R^2)$ be concentrated and uniformly distributed on the segment $[0,1] \times \{0\} \subset \R^2$, and let $\pi_s \in \mathcal{P}(\R^2)$ be concentrated and uniformly distributed on the segment $[0,1] \times \{s\} \subset \R^2$. Then, it holds that 
\begin{align*}
W_1(\pi_0,\pi_s) = |s|,
\end{align*}
for all $s \in \R$. In particular, the Wasserstein distance is sensitive to the relative position of the supports of singular distributions being compared. This allows one to better handle the vanishing gradient phenomena and ensure more stable learning using gradient-based algorithms \cite{wgan_main,wgan_gp}.
\subsubsection{The dual of the Kantorovich formulation of optimal transport}\label{sec:dualw}
Since the Kantorovich formulation of optimal transport is essentially an infinite-dimensional linear programming problem, it is plausible that it admits a strong dual formulation (see \cite{ekeland1999convex} for a general duality theory of convex variational problems). It can be shown that this is indeed the case.
The dual reformulation is known as the \emph{Kantorovich-Rubinstein (KR) duality}. If $c(x,y) \leq a(y) + b(x)$ for some suitable $a \in L^1(\Omega_1; \pi_1)$ and $b \in L^1(\Omega_2; \pi_2)$, the KR duality allows us to rewrite the Kantorovich formulation of optimal transport as 
\begin{align}\label{eq:dual}
    \sup_{(f,g) \in \Lambda(c)} \int_{\Omega_1} r(x)\, \mathrm{d}\pi_1(x) +  \int_{\Omega_2} s(y)\, \mathrm{d}\pi_2(y), 
\end{align}
where $\Lambda(c) = \{(r,s) : r\in L^1(\Omega_1; \pi_1), s \in L^1(
\Omega_2; \pi_2),\, r(x) + s(y) \leq c(x,y) \ \forall x,y\}$. Moreover, the supremum in \eqref{eq:dual} is attained and the optimal $r$ and $s$ are called Kantorovich potentials \cite[Theorem 1.17]{ambrosio2013user}. 
In the case where $\Omega_1 = \Omega_2 = \Omega$ and the cost is a metric (in which case we rename $c(x,y)$ by $d(x,y)$), then \eqref{eq:dual} can be rewritten as
\begin{align}\label{eq:dualwass}
     \sup_{g \in {\mathrm {Lip}}_1(\Omega)} \int_{\Omega} g(x)\, \mathrm{d}\pi_1(x) -  \int_{\Omega}g(y)\, \mathrm{d}\pi_2(y),
\end{align}
where ${\mathrm {Lip}}_1(\Omega)$ is the set of $1$-Lipschitz functions defined as 
\begin{align}\label{eq:dual2}
    {\mathrm {Lip}}_1(\Omega) = \left\{g : \Omega \rightarrow \mathbb{R} \ \ \text{s.t.}\ \ \sup_{x\neq y} \frac{|g(x) - g(y)|}{d(x,y)} \leq 1\right\}.
\end{align}

\subsubsection{Reconstructing the optimal transport map from the Kantorovich potentials}
Before concluding this section, we state an important result in optimal transport that, in certain cases, allows one to reconstruct the optimal transport map from the Kantorovich potential. Let $\pi_1$ and $\pi_2$ two probability measures on $(\Omega, \mathcal{B}(\Omega))$ where $\Omega \subset \R^d$ is compact with a boundary negligible with respect to the Lebesgue measure. Let $c: \Omega \times \Omega \rightarrow [0,\infty)$ defined as $c(x,y) = \eta(x-y)$ for $x, y\in \Omega$, where $\eta$ is a strictly convex function. Then, the following theorem relating the optimal transport map and the Kantorovich potential holds \cite[Theorem 1.17]{santambrogio2015optimal}.
\begin{theorem}\label{thm:brenier}
Suppose that $\pi_1$ is absolutely continuous with respect to the Lebesgue measure. Then there exists a Kantorovich potential $r$, and the optimal transport map $T$ can be reconstructed as 
\begin{align}\label{eq:formbrenier}
T(x) = x - (\nabla \eta)^{-1}(\nabla r(x)) \quad x \in \Omega.
\end{align}

\end{theorem}
This theorem is a consequence of more general results due to Brenier \cite{brenier1987decomposition} and
can be applied, for instance, for transport costs such as $c(x,y) = |x-y|^p$, where $1< p< \infty$. That is, this result applies to all $p$-Wasserstein distances with $1< p< \infty$. Specializing Theorem \ref{thm:brenier} to the cost $c(x,y) = \frac{1}{2}|x-y|^2$ corresponding to the $2$-Wasserstein distance, $T$ can be reconstructed as
\begin{align*}
T(x) = x - \nabla r(x) = \nabla \left(\frac{x^2}{2} - r(x)\right).
\end{align*}
Moreover, thanks to a consequence of the celebrated Brenier's theorem (e.g. see \cite[Proposition 1.21]{santambrogio2015optimal}) the function $u(x) = \frac{x^2}{2} - r(x)$ is convex, implying that the optimal transport map is the gradient of a convex function. This observation has been utilized to design machine learning methods to approximate optimal transport maps using gradients of convex functions \cite{makkuva2020optimal, amos2017input}.
\subsection{Convex analysis and monotone operator theory}


%
In this section, we will recall some classical results in convex analysis and monotone operator theory that will be useful for developing the theory behind many of the provable machine learning methods, as well as for gaining intuition behind their workings. In particular, in \Cref{sec:unsupervised}, we will link some of these classical results with the fixed point theory used in the learned iterative scheme setting.

We first recall some definitions and properties of convex functions and monotone operators, before presenting some examples of how we can leverage these regularity properties for convergence in the classical setting. We will also briefly discuss operator splittings, which serve as the basis for some learned iterative schemes. 

\subsubsection{Convex analysis}
We first review some common properties of convex functions that serve as common assumptions in the sequel. A more comprehensive overview of convex analysis can be found, for example, in \cite{rockafellar1997convex}. Let $\mathbb{X}$ be a Banach space and let $\mathbb{X}^*$ denote the corresponding dual space. We state some basic definitions in the following.
\begin{definition}
    A function $f:\mathbb{X} \rightarrow \overline{\R}$ is \emph{lower semi-continuous} (l.s.c.) at a point $x$ if for every sequence $x_n \rightarrow x$ in $\mathbb{X}$, 
    \[\liminf_{n \rightarrow \infty}{f(x_n)} \ge f(x_0).\]
    $f$ is \emph{proper} if the \emph{effective} domain
    \[\dom f \coloneqq \{x \in \mathbb{X} \mid f(x) < +\infty\}\]
    is nonempty. We define $\Gamma_0(\mathbb{X})$ as the class of convex proper l.s.c. functions from $\mathbb{X} \rightarrow \overline{\R}$, and drop the argument when the domain is understood from the context.

    \noindent $f$ is coercive if for all $x_n$ with $\|x_n\|\rightarrow \infty$, we have that $f(x_n) \rightarrow \infty$. 
\end{definition}
The class $\Gamma_0$ gives sufficient conditions for many regularity conditions to hold, and will be the main assumption for our objective optimization problems. Lower semi-continuity may sometimes be equivalently referred to as \emph{closed} in the literature, which is shown in the following proposition.
\begin{proposition}[{\cite[Prop. 2.3]{ekeland1999convex}}]
    For a function $f:\mathbb{X} \rightarrow \overline{\R}$, the following two statements are equivalent:
    \begin{enumerate}
        \item $f$ is lower semi-continuous;
        \item $f$ is \emph{closed}, i.e., the epigraph $\text{epi}(f):=\{(x,t) \in \mathbb{X} \times \R \mid t \ge f(x)\}$ is closed.
    \end{enumerate}
\end{proposition}
These properties can be used directly to show that a convex function has a minimizer.

\begin{theorem}
    Let $\mathbb{X}$ be a Banach space and $\tau_{\mathbb{X}}$ be some topology on $\mathbb{X}$ such that bounded sequences in $\mathbb{X}$ have $\tau_{\mathbb{X}}$-convergent subsequences. If $f:\mathbb{X} \rightarrow \overline{\R}$ is proper, bounded from below, coercive, and $\tau_{\mathbb{X}}$-l.s.c., then $f$ has a minimizer.
\end{theorem}
\begin{proof}
    By boundedness from below, a minimizing sequence $(x_n)$ exists. By coercivity, the minimizing sequence is bounded. Apply the condition on $\tau_{\mathbb{X}}$ to a minimizing sequence to obtain a limit in $\tau_{\mathbb{X}}$, $x_{k_n} \rightarrow \bar{x}$. By lower semi-continuity, the limit satisfies $f(\bar{x}) \le \underset{n}{\liminf} f(x_{k_n}) = \underset{\mathbb{X}}{\inf}\, f$. Therefore, $\bar{x}$ is a minimizer of $f$.
\end{proof}

We continue with some definitions extending the classical notion of a derivative to convex non-differentiable functions.

\begin{definition}
    A function $f:\mathbb{X} \rightarrow \overline{\R}$ is \emph{subdifferentiable} at a point $u \in \mathbb{X}$ if there exists a dual element $p \in \mathbb{X}^*$ such that 
    \[f(v) \ge f(x) + \langle p, v-x \rangle,\ \forall v \in \mathbb{X}.\]
    The dual element $p$ is called a \emph{subgradient} at $u$. The \emph{subdifferential} of $f$ at $u$, denoted $\partial f(u)$, is the collection of all such subgradients of $f$ at $u$, i.e.
    \[\partial f(u) \coloneqq \{p \in \mathbb{X}^* \mid f(v) \ge f(x) + \langle p, v-x \rangle,\ \forall v \in \mathbb{X}\}.\]
\end{definition}

The subdifferential is a multi-valued operator that shares some properties with the classical derivative operator, which can additionally be defined for discontinuous functions. In particular, if $f$ is differentiable at a point $u$, then the subdifferential is equal to the singleton set containing the derivative $\partial f(u) = \left\{f'(u)\right\}$. The following propositions state sufficient (but not necessary) conditions for the existence of the subdifferential, as well as some useful properties. Additional properties can be found in classical literature \cite{rockafellar1997convex,bauschke2011convex,charalambos2013infinite}.

\begin{proposition}[{\cite{phelps2009convex,ekeland1999convex}}]
    Suppose $f:\mathbb{X} \rightarrow \overline{\R}$ is convex, finite, and continuous at some $u\in X$. Then $\partial f(v) \neq \emptyset$ for all $v \in \mathbb{X}$. Moreover, $0 \in \partial f(v)$ if and only if $v$ is a minimizer of $f$. If $g:\mathbb{X} \rightarrow \overline{\R}$ is another convex proper l.s.c. function and $f$ is continuous at some $u \in \dom f \cap \dom g$, then
    \[\partial(f+g) = \partial f + \partial g.\]
    If $f$ is instead G\^ateaux differentiable at $u \in \mathbb{X}$, then it is subdifferentiable at $u$ and $\partial f(u) = \{f'(u)\}$.
\end{proposition}
\begin{proposition}[{\cite[Thm. 7.13]{charalambos2013infinite}}]
    Suppose $f:\mathbb{X} \rightarrow \overline{\R}$ is proper and convex. If $u \in \dom f$, then $\partial f(u)$ is convex and weak-* compact.
\end{proposition}

Under these conditions, we can define the proximal operator, which tries to move towards the minimizer of $f$, regularized by the distance to the initial point.

\begin{definition}[Proximal operator]
    For a convex function $f:\mathbb{X} \rightarrow \overline{\R}$, the \emph{proximal operator} is defined as
    \begin{equation}
        \prox_{f}(x) = \argmin_{y \in \mathbb{X}} \left\{\frac{1}{2} \|y-x\|^2 + f(y)\right\}
    \end{equation}
\end{definition}
The following proposition details some properties of the proximal operator. In particular, it can be thought of as an implicit Euler discretization of gradient flow, as opposed to the explicit Euler discretization that is gradient descent.
\begin{proposition}[{\cite{rockafellar1997convex, ekeland1999convex, rockafellar2009variational}}]
    For a proper convex l.s.c. function $f$, the proximal operator is well-defined and is single-valued. Moreover, it satisfies the following:
    \begin{enumerate}
        \item $\prox_f$ is nonexpansive, and, in particular, is continuous. 
        \item Fixed points of $\prox_f$ correspond to minimizers of $f$: 
            \[\{x_0 \in {\mathbb{X}}\mid x_0 = \prox_f(x_0)\} = \argmin_{\mathbb{X}} f.\]
    \end{enumerate}
    
    If $\mathbb{X} = \R^n$, the following also hold:
    
    \begin{enumerate}
        \item (Moreau's identity) $\prox_f + \prox_{f^*} = \id$, where $\id$ is the identity map on $\mathbb{X}$; 
        \item Letting the Moreau envelope be defined as 
        \begin{equation}
            M_{\lambda f}(x) = \inf_{y \in \R^n} \left\{f(x) + \frac{1}{2\lambda} \|x-y\|_2^2\right\},
        \end{equation}
        the proximal operator satisfies
        \begin{equation}
            \nabla M_{\lambda f}(x) = \frac{1}{\lambda}(x - \prox_{\lambda f}(x)) = \prox_{f^*/\lambda}(x).
        \end{equation}
        \item $\partial f$ and $\prox_f$ are maximally monotone mappings (see \Cref{sec:monotone_op} for the definition) from $\R^n$ to $\R^n$. 
    \end{enumerate}
\end{proposition}
\subsubsection{Monotone operator theory}
\label{sec:monotone_op}
A common way of showing the convergence of some iterative methods is through monotone operator theory, consisting of fixed-point results. Monotone operators are inextricably tied to convex functions through the proximal and subgradient operators, making them a useful tool for showing convergence within the realm of convexity.
\begin{definition}[Monotonicity]
    A set-valued mapping $T:\R^n \rightrightarrows \R^n$ is $\emph{monotone}$ if for all $x, x' \in \R^n,\, p \in T(x),\, p' \in T(x')$,
    \[\langle p-p', x-x' \rangle \ge 0,\]
    and \emph{strictly monotone} if the inequality is strict for $x \ne x'$. The \emph{resolvent} of $T$ is the operator
    \[J_T \coloneqq (\id + T)^{-1},\]
    and the \emph{reflected resolvent} is
    \[R_T \coloneqq 2J_T - \id.\]
    $T$ is said to be \emph{maximally monotone} if its graph $G(T) = \{(x, p): x \in \R^n,\, p \in T(x)\}$ is not contained within the graph of another monotone operator.
\end{definition}
The convex minimization problem $\underset{{x\in X}}{\min}\, f(x)$ thus corresponds to the monotone inclusion problem $0 \in \partial f(x)$. In the differentiable case, this resolves to solving the optimality condition $f'(v) = 0$. Moreover, for $f \in \Gamma_0$, the proximal operator is the resolvent of the subgradient operator, i.e. $\prox_f = J_{\partial f}$. Monotonicity is intrinsically related to convexity as described in the following theorem. Intuitively, it means that sub-gradients are aligned with ascent directions. One important concept is that of non-expansiveness, which is crucial in the study of fixed-point convergence.

\begin{definition}[Non-expansiveness]
    A mapping $T:\R^n \rightarrow \R^n$ is \emph{non-expansive} if for all $x,y \in \R^n$,
    \[\|T(x) - T(y)\| \le \|x-y\|.\]
    $T$ is \emph{firmly non-expansive} if for all $x,y \in \R^n$,
    \[\|T(x) - T(y)\|^2 + \|(\id - T)(x) - (\id - T)(y)\|^2 \le \|x-y\|^2.\]
    Note that firm non-expansiveness implies non-expansiveness.
\end{definition}

The concepts of monotonicity and non-expansiveness are intrinsically tied to convexity, as shown by the following results.

\begin{theorem}[{\cite[Sec. 12.C.]{rockafellar2009variational}}]
    For a proper l.s.c. function $f:\R^n \rightarrow \overline{\R}$, $f$ is convex if and only if $\partial f$ is monotone, in which case $\partial f$ is also maximally monotone. Moreover, for any $\lambda>0$, the proximal mapping $\prox_{\lambda f}:\R^n \rightrightarrows \R^n$ is monotone. If $f$ is additionally convex, then $\prox_{\lambda f} = J_{\lambda \partial f}$ is maximally monotone and also non-expansive.
\end{theorem}
\begin{proposition}[{\cite[Cor. 23.10]{bauschke2011convex}}]
    For a maximally monotone operator $A:\R^n \rightrightarrows \R^n$, we have that 
    \begin{enumerate}
        \item $J_A$ and $\id - J_A$ are firmly nonexpansive and maximally monotone;
        \item $R_A$ is non-expansive.
    \end{enumerate}
\end{proposition}

\subsubsection{Operator splitting}
Convex optimization problems are often solved using iterative methods, where a sequence is constructed that converges to the minimizer, with some common methods including subgradient descent or proximal gradient descent. Recall that variational problems typically take the form of a composite optimization problem \eqref{eq:var_reg}. Usually, the fidelity and regularization terms will have different regularity conditions, such as smoothness or Lipschitz conditions. We can exploit the composite structure to simplify each iteration. Noting the correspondence between convex problems and monotone inclusion problems, we can convert the above problem to finding the equivalent problem of finding zeros of sums of two maximally monotone operators. 

\noindent Consider the following inclusion problem
\[0 \in Ax + Bx,\]
where $A$ and $B$ are both maximally monotone operators, which arises naturally from finding a minimizer of the sum of two convex functions. If $A+B$ is also maximally monotone, then one possible approach is to consider root solving using the resolvent $J_{\gamma(A+B)}$. However, this is generally difficult to compute, for example in the case when $A, B$ are proximal operators of convex functions $f, g \in \Gamma_0$, respectively. Therefore, one seeks to find a zero of $A+B$, using only their resolvents $J_{\gamma A}$ and $J_{\gamma B}$. This process of splitting the resolvent of $A+B$ into the resolvents of its components is generally referred to as a splitting algorithm and can be performed in different ways \cite{lions1979splitting}. We present two simple versions, which are by far the most widely used splitting techniques in convex optimization: the forward-backward splitting and the Douglas-Rachford splitting \cite{douglas1956numerical}. 
    \begin{theorem}[{Douglas-Rachford Splitting \cite[Thm. 25.6]{bauschke2011convex}}]
        For a Hilbert space $\mathcal{H}$, let $A, B:\mathcal{H} \rightrightarrows \mathcal{H}$ be maximally monotone operators such that $\zer(A+B) \ne \emptyset$. Let $(\lambda_n)_{n\in \mathbb{N}}$ be a sequence in $[0,2]$ satisfying $\sum_{n\in \mathbb{N}} \lambda_n(2-\lambda_n) = +\infty$, $\gamma>0$ be a step-size. Let $x_0 \in \R^n$ be an initialization. Consider the iterations 
        \begin{equation}\tag{DRS}
            \begin{cases}
                y_n = J_{\gamma B} x_n,\\
                z_n = J_{\gamma A} (2y_n - x_n),\\
                x_{n+1} = x_n + \lambda_n (z_n-y_n).\\
            \end{cases}
        \end{equation}
        Then there exists a fixed point $x \in \Fix R_{\gamma A}R_{\gamma B}$ such that the following hold:
        \begin{enumerate}
            \item $J_{\gamma B}(x) \in \zer(A+B)$
            \item $y_n-z_n$ converges strongly to zero,
            \item $x_n$ converges weakly to $x$
            \item $y_n$ and $z_n$ converge weakly to $J_{\gamma B}(x)$.
        \end{enumerate}
    \end{theorem}
    Note that in the case where the Hilbert space $\mathcal{H}$ is finite-dimensional, weak convergence is equivalent to strong convergence. Letting $A$ and $B$ be proximal operators of some proper convex l.s.c. functions $f$ and $g$, we get convergence to a fixed point of $\prox_{f+g}$, using only proximal operators or subgradients of $f$ and $g$ separately. Further, the fixed point is a minimum of $f+g$. This is particularly useful wherein $f$ and $g$ have easy-to-compute proximals, while $f+g$ does not.

    By casting the above monotone inclusion problem in the scope of convex functions, with $A$ being a derivative and $B$ being a proximal operator, we can obtain splitting schemes that optimize the sum of two convex functions, where one of the functions is smooth. 

    \begin{theorem}[{Forward-Backward Splitting \cite[Cor. 27.9]{bauschke2011convex}}]\label{thm:FBS}
        Let $f:\mathcal{H}\rightarrow \R$ be convex and differentiable with $1/\beta$-Lipschitz gradient, and $g:\mathcal{H}\rightarrow \R$ be proper convex l.s.c. and possibly non-smooth. Let $\gamma \in (0,2\beta)$ and set $\delta = \min\{1,\,\beta/\gamma\}+1/2$. Further let $(\lambda_n)_{n \in \mathbb{N}}$ be a sequence in $[0,\delta]$ such that $\sum_{n \in \mathbb{N}} \lambda_n(\delta - \lambda_n) = +\infty$. Suppose that $f+g$ admits a minimizer and let $x_0 \in \mathcal{H}$. Then, the forward-backward iterations, given by 
        \begin{equation}\tag{FBS}
            \begin{cases}
                y_n = x_n - \gamma \nabla f(x_n), \\
                x_{n+1} = x_n + \lambda_n(\prox_{\gamma g}y_n - x_n),
            \end{cases}
        \end{equation}
        satisfy the following:
        \begin{enumerate}
            \item $(x_n)_{n \in \mathbb{N}}$ converges weakly to a point in $\argmin_{\mathcal{H}}(f+g)$;
            \item Suppose $\inf_n \lambda_n>0$ and $x \in \argmin_{\mathcal{H}}(f+g)$. Then $\nabla f(x_n)$ converges strongly to $\nabla f(x)$.
        \end{enumerate}
    \end{theorem}

    Note that by taking $\lambda_n = 1$, the FBS algorithm alternates between a proximal step on $g$ and a gradient descent step on $f$. Optimizing the sum of two convex functions where one is smooth arises naturally in variational regularization \cite{scherzer2009variational,kaipio2006statistical}. In this case, $f$ is usually chosen to be a smooth fidelity term, such as the $\ell_2^2$ penalty. As this minimization is usually ill-posed, a regularization term $g$ is added to the fidelity, representing a prior that is imposed on the data. 

\begin{example}[ISTA]
    Consider the case where our Hilbert space is finite-dimensional Euclidean space $\mathcal{H} = \R^n$. Let $A:\R^n \rightarrow \R^m$ be a bounded linear operator, and let $z \in \R^m$. The iterative shrinkage thresholding algorithm (ISTA) considers the optimization problem where $f(x) = \|Ax-z\|^2/2$, with an $\ell_1$ regularization \cite{daubechies2004iterative,gregor2010learning,beck2009fast}. This is used in sparse coding, where the $\ell_1$ penalty enforces sparsity on $x$ and is sometimes referred to as LASSO regression in the statistical literature \cite{tibshirani1996regression}. The resulting optimization problem to solve is
    \begin{equation*}
        \argmin_{x \in \R^n} f(x) + g(x) \coloneqq \frac{1}{2}\|Ax - z\|_2^2+ \lambda \|x\|_1,
    \end{equation*}
    where $\lambda>0$ is a regularization parameter. Note that this $f+g$ admits a minimizer since it is coercive and bounded below by 0. Moreover, $g$ is not differentiable, so first-order methods that rely on the gradient of $f+g$ are not applicable. We can, however, apply \Cref{thm:FBS} to obtain a (strongly) convergent scheme. We first observe that the proximal operator of $\lambda \|x\|_1$ is the coordinate-wise shrinkage operator, defined by 
    \begin{equation*}
        \prox_{\alpha \|x\|_1} = h_{\alpha}(x),\,\,\,\text{where}
    \end{equation*}
    \begin{equation*}
        [h_\alpha (x)]_i = \sign(x_i) \max(|x_i| - \alpha, 0).
    \end{equation*}
Observe that the proximal operator of $\|\cdot \|_1$ is straightforward to compute. Taking the step-sizes $\lambda_n = 1$ and $\gamma < 1/\|A^\top A\|$, ISTA reduces to the forward-backward scheme
    \begin{equation}\tag{ISTA}
            x_{n+1} = h_{(\lambda \gamma)}\left(x_n - \gamma \nabla f(x_n)\right).
    \end{equation}
\end{example}

\subsubsection{Pseudo-inverses}\label{sec:pseudoinverse}
For two Banach spaces $\mathbb{X}$ and $\mathbb{Y}$, a bounded linear operator $A \in \mathcal{L}(\mathbb{X}, \mathbb{Y})$ may not be invertible in the usual sense outside the range of $A$. Recall that for a bounded linear operator $A$, its null-space $\ker(A)$ is closed, and thus admits a unique orthogonal complement in $\mathbb{X}$. Moreover, $A$ restricted to $\ker(A)^\perp$ is injective and thus admits a linear inverse from $\range(A)$ to $\ker(A)^\perp$.

\begin{definition}
    For a linear operator $A\in \mathcal{L}(\mathbb{X}, \mathbb{Y})$, let $\tilde{A}$ denote the restriction of $A$ to $\ker(A)^\perp \subseteq \mathbb{X}$, where $\ker(A)$ is the null-space of $A$. Note that $\tilde{A}$ is invertible. The \emph{Moore-Penrose pseudo-inverse} $A^\dagger:\mathcal{D}(A^\dagger) \rightarrow \mathbb{X}$ is the unique linear extension of $\tilde{A}^{-1} : \range(A) \rightarrow \ker(A)^\perp$ to the domain
    \[\mathcal{D} (A^\dagger) \coloneqq \range(A) \oplus \range(A)^\perp,\]
    satisfying $\ker(A^\dagger) = \range(A)^\perp$.
\end{definition}
\begin{remark}
    If $\mathbb{X},\mathbb{Y}$ are finite-dimensional, $\mathcal{D}(A^\dagger) = \mathbb{Y}$. The Moore-Penrose pseudo-inverse is equivalent to linearly extending the inverse of $A$ from $\ker(A)^\perp$ to all of $\mathbb{Y}$ by defining $A^\dagger: \ker(A) \mapsto 0$. 
\end{remark}

\begin{proposition}[{\cite[Prop. 2.3]{engl1996regularization}}]
    The Moore-Penrose pseudo-inverse satisfies the following properties:
    \begin{enumerate}
        \item $A^\dagger A = \Pi_{\ker(A)^\perp}$,
        \item $A A^\dagger = \Pi_{\overline{\range(A)}} \vert_{\mathcal{D}(A^\dagger)}$,
        \item $A A^\dagger A = A$,
        \item $A^\dagger A A^\dagger = A^\dagger$.
    \end{enumerate}
\end{proposition}
The Moore-Penrose inverse is not necessarily continuous. It is continuous if and only if $\range(A)$ is closed {\cite{engl1996regularization}}. Moreover, it can be very ill-conditioned if $A$ has small singular values. This explains why the direct inversion of inverse problems is unstable, necessitating the use of regularization techniques.

\subsection{Supervised versus unsupervised learning of reconstruction operators}\label{sec:learningStyles}
In this section, we outline different training strategies for learning a data-driven reconstruction operator for imaging inverse problems based on available training data. The specific training strategy adopted for a given problem depends on several practical considerations, such as the type and amount of available data, computational requirements, desired theoretical guarantees, etc. In general, supervised approaches tend to result in better empirical performance than unsupervised approaches, but it might be infeasible to acquire paired trained data for supervised learning in problems of practical interest.   
\subsubsection{Supervised learning}
In supervised learning, one seeks to learn a reconstruction map $G_{\theta}:\mathbb{Y}\rightarrow \mathbb{X}$, typically parameterized using a deep neural network (DNN), utilizing pairs of training examples $(x^{(i)},y^{(i)})_{i=1}^{N}$ drawn from the (unknown) joint density distribution of the $(\mathbb{X}\times \mathbb{Y})$-valued random variable $(\stx,\sty)$, where $\sty = A\stx+\stw$. The parameter $\theta$ is learned by minimizing the empirical reconstruction error measured using a suitable loss functional $\ell:\mathbb{X}\times \mathbb{X}\rightarrow \mathbb{R}^+$ over the training data set: 
\begin{equation}
    \theta^*\in\underset{\theta}{\argmin}\,J(\theta),\text{\,\,where\,\,}J(\theta):=\frac{1}{N}\sum_{i=1}^N \ell\left(x^{(i)},G_{\theta}(y^{(i)})\right).
    \label{eq:sup_training}
\end{equation}
The key challenge in supervised learning is to construct a suitable parameterization of the reconstruction operator $G_{\theta}$ such that it is sufficiently expressive and encodes knowledge about the data generation process (i.e., the forward operator $A$). To this end, several techniques have been proposed achieving remarkable performances in inverse problems reconstruction  \cite{zhu2018image, oh2018eter, postprocessing_cnn, chen2017low, gilton2019neumann, Schwab_2019_null_space, eldar_spm}. Here, we describe two specific ones that are relevant for the unsupervised methods treated in this chapter: (i)  \textit{post-processing} approaches \cite{postprocessing_cnn} and (ii) \textit{algorithm unrolling} (see \cite{eldar_spm} and references therein). The post-processing approach consists in designing $G_{\theta}$ as the composition $G_{\theta}=\mathcal{C}_{\theta}\circ \rho$, where a model-based reconstruction operator $\rho:\mathbb{Y}\rightarrow \mathbb{X}$ (e.g., the filtered back-projection (FBP) in CT) is followed by a deep convolutional neural network (CNN) $\mathcal{C}_{\theta}:\mathbb{X}\rightarrow \mathbb{X}$  that is trained to remove artifacts from $\rho(y)$. Since post-processing approaches do not fully incorporate the physics of the imaging system, they typically need large amounts of training data to generalize well on unseen data. Moreover, the final reconstructed image produced by a post-processing method does not necessarily satisfy \textit{data-consistency}, meaning that a small value of the fidelity $\ell(y,A\,\rho(y))$ corresponding to $\rho$ does not imply a small value of the fidelity $\ell(y,A\,\mathcal{C}_{\theta}(\rho(y)))$. 


The algorithm unrolling framework offers a more principled approach for incorporating imaging physics into the reconstruction operator. As the name suggests, algorithm unrolling builds the reconstruction operator by first \textit{unfolding} a small number of iterations of an optimization algorithm (such as proximal gradient descent (PGD)) for solving the variational image reconstruction problem \eqref{eq:var_reg}, and then by replacing the components that do not depend on the imaging process using learnable data-driven units. In the interest of concreteness, consider \eqref{eq:var_reg} where both $f$ and $R_{\lambda}$ are in $\Gamma_0(\mathbb{X})$, $\nabla f$ is $L_{\nabla f}$-Lipschitz continuous, but $R_{\lambda}$ is not necessarily differentiable. If $R_{\lambda}$ admits a cheaply computable proximal operator, a natural choice for solving \eqref{eq:var_reg} is the PGD algorithm given by
\begin{equation}
    x_{k+1} = \text{prox}_{\eta R_{\lambda}}\left(x_k - \eta\,\nabla f(y,Ax_k)\right), \,k=0,2,\cdots, N-1,
    \label{eq:basic_pgd}
\end{equation}
where $\eta\leq \frac{1}{L_{\nabla f}}$. For large-scale image reconstruction problems such as medical imaging, one would typically need a few thousand iterations of PGD to obtain a reasonable reconstruction, which could be unacceptably slow. The key idea behind algorithm unrolling is to truncate iterative optimization algorithms such as \eqref{eq:basic_pgd} after a small number of iterations (for example, $N\sim 10$), and replace the proximal operator with a CNN $\psi_{\theta_k}:\mathbb{X}\rightarrow \mathbb{X}$ for each $k$. The parameters $\theta=(\theta_k)_{k=1}^{N}$ are then learned by minimizing the empirical risk $J(\theta)$ on the training data set:
\begin{equation}
    J(\theta):=\frac{1}{N}\sum_{i=1}^{N}\ell\left(x^{(i)},x_N^{(i)}(\theta)\right),
    \label{eq:loss_unrolled_pgd}
\end{equation}
where $x_{k+1}^{(i)}:=\psi_{\theta_k}\left(x_k^{(i)} - \eta\,\nabla f\left(y^{(i)},Ax_k^{(i)}\right)\right), \,k=0,2,\cdots, N-1$. The origin of algorithm unrolling can be traced back to the seminal work by Gregor and LeCun on learned iterative shrinkage thresholding algorithms (LISTA) \cite{gregor2010learning} for efficient sparse coding. In recent years, such methods have been extensively developed and they currently offer performances able to achieve the state-of-art for supervised inverse problems reconstruction. We refer the interested reader to \cite{lpd_tmi, gilton2021deep, wu2019computationally, chun2020momentum, mehranian2020model,tang2022accelerating} and the references therein for further details on algorithm unrolling.
\subsubsection{Unsupervised learning}\label{sssec:unsupervised}
In contrast to \textit{supervised learning}, we will use the phrase \textit{unsupervised learning} to refer to any scenarios where one does not have access to paired training examples drawn from the joint distribution of $(\stx,\sty)$, but only on the marginal distributions of $\stx$ and $\sty$. From a practical perspective, unsupervised learning approaches are more realistic in real-world applications, as it is generally challenging to acquire paired examples for training reconstruction operators. For instance, the training data set in the image reconstruction problem in X-ray CT consists of high-quality reconstructed images $x^{(i)}$ obtained from high- or normal-dose projection data, and their corresponding low-dose projection data $y^{(i)}$. This is generally difficult to obtain, as it necessitates scanning a large number of subjects with two different doses, then aligning the respective scans voxel-wise to ensure exact correspondence between $x^{(i)}$ and $y^{(i)}$. 

Broadly, one might encounter the following three scenarios (or, some combinations thereof) in unsupervised learning so far as the training data is concerned. 
\begin{enumerate}[leftmargin=*]
    \item \textbf{Unpaired training examples}: In this setting, the training data consists of i.i.d. samples $(x^{(i)})_{i=1}^{N_1}$ and $(y^{(j)})_{j=1}^{N_2}$ drawn from the marginal distributions $\pi_{\stx}$ and $\pi_{\sty}$ of the ground-truth images and the measured data, respectively. Using only the knowledge of the marginal distributions $\pi_{\stx}$ and $\pi_{\sty}$, one aims at learning a correspondence between the probability distributions in the form of a reconstruction $G_\theta : \mathbb{Y} \rightarrow \mathbb{X}$ such that $(G_\theta)_{\#} \pi_{\sty} = \pi_{\stx}$. Additionally the reconstruction needs to satisfy \emph{data-consistency}, meaning that $y$ is close to $AG_\theta(y)$ for most of the samples $y$ from the marginal $\pi_{\sty}$.
    In Section \ref{sec:OTapproaches}, we will describe in details several approaches using unpaired training samples that are based on optimal transport techniques and cycle architectures. However, we point out that many other unsupervised methods based have been proposed in the literature. Such approaches are often based on conditional variants of generative models and on their inversion. We refer the interested reader to 
     \cite{wolterink2017generative, winkler2019learning, batzolis2021caflow, lugmayr2020srflow, mirza2014conditional, saharia2022palette, zhou2021cocosnet} and the references therein.\\
    \item \textbf{Learning the prior}: In many applications, one has only access to samples $(x^{(i)})_{i=1}^{N_1}$ from the distribution $\pi_{\stx}$ of the ground-truth images. In such cases, the primary objective is to utilize ideas from the generative machine learning approaches (such as generative adversarial networks (GANs), variational autoencoders (VAEs), etc.) to build a reasonable approximation of the image prior to regularize the inverse problem. Many approaches have been proposed to achieve this goal, based, for instance, on constructing a projection on the range of the pre-trained generator and approximating its inverse \cite{asim2020invertible, shah2018solving, xia2022gan, daras2021intermediate, bora2017compressed}.
    Plug-and-play (PnP) denoising methods (which we review in \Cref{sec:Reg_byPnP}) also fall in this category as they seek to implicitly learn a regularizer through an image denoiser. 
    \\  
    \item \textbf{Fully unsupervised approaches}: We will use this term to refer to the case where only i.i.d. samples $(y^{(j)})_{j=1}^{N_2}$ from the data distribution $\pi_{\sty}$ are available for training. These methods are essentially ground-truth-free, as do not make use of the true images during training. Among various approaches in this category, we provide a detailed treatment of the emerging \textit{learning-to-optimize} paradigm in \Cref{ssec:l2o}. These methods seek to learn a fast solver for high-dimensional convex optimization problems that arise frequently in inverse problems by leveraging training data (while not utilizing any ground-truth). Some notable methods in this category (such as unbiased risk estimation, deep image prior, equivariance, etc.) are briefly reviewed in \Cref{ssec:GTfreeIR}. 
\end{enumerate}

\section{Optimal transport-based unsupervised approaches}\label{sec:OTapproaches}
In recent years, optimal transport-based methods have been extensively used to address unsupervised data-driven tasks such as image generation \cite{wgan_main,wgan_gp}, domain adaptation, image-to-image translation, and image super-resolution. Unsurprisingly, many inverse problems in areas such as medical imaging, geophysics, and fluid dynamics have benefited from such methods in terms of both modeling capabilities and the efficiency of the available algorithms. In the subsequent sections, we will illustrate several optimal transport-based unsupervised approaches for inverse problems, ultimately aiming to draw a connection between them.

\subsection{Cycle-GAN--based approaches to unsupervised learning}
We start by addressing methods that are based on cyclic models. Inspired by Cycle-GAN \cite{cyclegan_zhu2020unpaired}, such approaches are particularly suited for inverse problems in the case of unsupervised data since they allow enforcing a coupling between ground-truth images and measurements through a cycle-consistency penalty. Optimal transport metrics have been incorporated into these models, allowing for more stable training.

\subsubsection{Wasserstein generative adversarial networks (WGANs)}
Before addressing cycle-based approaches, we recall classical generative models, with a particular focus on the ones based on optimal transport techniques. Wasserstein generative adversarial networks (WGANs) \cite{wgan_main,wgan_gp} have incorporated optimal transport techniques for image generation, achieving performance superior to that of traditional generative adversarial networks (GANs) \cite{goodfellow2014generative}, while ensuring a more stable training for high dimensional data-sets while mitigating the problem of mode-collapse \cite{che2016mode}. Denoting by $\pi_\mathbf{v} \in \mathcal{P}(\mathbb{V})$ a known latent distribution in $\mathbb{V}$ (which can be easily sampled) and $\pi_\mathbf{x} \in \mathcal{P}(\mathbb{X})$ the unknown ground-truth distribution, Wasserstein GANs aim to construct a generator $G_\theta: \mathbb{V} \rightarrow \mathbb{X}$ by minimizing the $1$-Wasserstein distance between $(G_\theta)_{\#}\pi_{\mathbf{v}}$ and $\pi_{\mathbf{x}}$, i.e.
\begin{align}\label{eq:wgan}
\min_\theta  W_1\left((G_\theta)_{\#}\pi_{\mathbf{v}}, \pi_{\mathbf{x}}\right),
\end{align}
where $G_\theta$ is typically parameterized by a suitable DNN.
Applying the dual formulation of the $1$-Wasserstein distance, c.f. \eqref{eq:dual}, the objective in \eqref{eq:wgan} can be equivalently rewritten as 
\begin{align}\label{eq:wgan2}
    \min_\theta \sup_{g \in {\mathrm {Lip}}_1(\mathbb{X})} \int_{\mathbb{X}} g(x)\, \mathrm{d}(G_\theta)_{\#}\pi_{\mathbf{v}} - \int_{\mathbb{X}} g(x)\, \mathrm{d}\pi_{\mathbf{x}}.
\end{align}
By expressing the constraint $g \in {\mathrm {Lip}}_1(\mathbb{X})$ as a penalization in the objective, and applying the definition of push-forward of probability measures, \eqref{eq:wgan2} can be approximated by the following min-max problem 
\begin{align}\label{eq:obwgan}
    \min_\theta \,\sup_{\sigma} \int_{\mathbb{V}} g_\sigma(G_\theta(v))\, \mathrm{d}\pi_{\mathbf{v}} & - \int_{\mathbb{X}} g_\sigma(x)\, \mathrm{d} \pi_{\mathbf{x}} + \lambda \int_{\mathbb{X}} (|\nabla g_\sigma|(\hat x) - 1)^2_+\, \mathrm{d}\hat \pi,
\end{align}
where $g_\sigma: \mathbb{X} \rightarrow \R$ is parametrized by a suitable DNN, $\lambda >0$ is a positive parameter and $\hat \pi \in \mathcal{P}(\mathbb{X})$ is defined by sampling uniformly on the lines connecting samples of $\pi_{\mathbf{x}}$ and samples of $(G_\theta)_{\#}\pi_{\mathbf{v}}$. The network $g_{\sigma}$ is referred to as the \textit{discriminator} or the \textit{critic}, since, during training, it learns to tell apart the ground-truth images from the generated ones. The training is performed by optimizing \eqref{eq:obwgan} computed on the empirical approximation $\pi_{\mathbf{x}} \sim \frac{1}{N_1} \sum_{i=1}^{N_1} \delta_{x^{(i)}}$ and $\pi_{\mathbf{v}} \sim \frac{1}{N_2} \sum_{i=1}^{N_2} \delta_{v^{(i)}}$, where $(x^{(i)})_{i=1}^{N_1}$ are the training samples and $(v^{(i)})_{i=1}^{N_2}$ are samples drawn from $\pi_{\mathbf{v}}$. 
From a theoretical point of view, the objective of WGAN closely resembles the classical GAN objective
\begin{align}\label{eq:gan}
    \min_\theta\,\, \sup_\sigma \int_{\mathbb{V}} \log(1 - d_\sigma (G_\theta(v))\, \mathrm{d}\pi_{\mathbf{v}} + \int_{\mathbb{X}} \log(d_\sigma(x))\, \mathrm{d}\pi_{\mathbf{x}},
\end{align}
where $G_\theta: \mathbb{V} \rightarrow \mathbb{X}$ is the generator and $d_\sigma : \mathbb{X} \rightarrow \R$ is the discriminator. Indeed, \eqref{eq:wgan} and \eqref{eq:gan} are both expressed as an adversarial min-max problem, with the substantial difference that optimizing \eqref{eq:gan} is equivalent to minimizing the Jensen-Shannon divergence between $\pi_{\mathbf{x}}$ and $(G_\theta)_\#\pi_{\mathbf{v}}$. Since WGAN aims to minimize the $1$-Wasserstein distance, the considerations of Section \ref{sec:notionsdistances} apply, justifying why WGAN is more stable for learning high-dimensional data distributions supported on lower dimensional manifolds \cite{wgan_main}.

We conclude this section by mentioning that many optimal transport-based generative models besides the WGAN framework are available in the literature. We will not focus on them here; however, we refer the interested reader to \cite{patrini2020sinkhorn, tolstikhin2017wasserstein, genevay2018learning, wu2019sliced} and the references therein.
\subsubsection{Cycle-GAN--based approaches for inverse problems}
Classical GANs and WGANs are both characterized by the simultaneous training of a generator $G_\theta: \mathbb{V} \rightarrow \mathbb{X}$ mapping a low-dimensional latent space to a high-dimensional data space, and a discriminator mapping $\mathbb{X}$ to $\R$. Cycle-GAN was introduced in \cite{cyclegan_zhu2020unpaired} to address unsupervised image-to-image translation between two data sets in $\mathbb{X}$ and $\mathbb{Y}$. This has been achieved by coupling the action of two generators $H_\sigma: \mathbb{X} \rightarrow \mathbb{Y}$ and $G_\theta: \mathbb{Y} \rightarrow \mathbb{X}$ that are trained to achieve cycle-consistency by enforcing that $H_\sigma (G_\theta(y)) \approx y$ and $G_\theta (H_\sigma(x)) \approx x$ for samples in $\pi_{\mathbf{y}}$ and $\pi_{\mathbf{x}}$, where $\pi_{\mathbf{x}}$ and $\pi_{\mathbf{y}}$ are the data distributions in $\mathbb{X}$ and $\mathbb{Y}$ respectively. Moreover, the generators are trained together with two discriminators $d^\mathbb{X}_{\tilde \theta} : \mathbb{X} \rightarrow \R$   and $d^\mathbb{Y}_{\tilde \sigma} : \mathbb{Y} \rightarrow \R$ designed to ensure that $(H_\sigma)_\# \pi_{\mathbf{x}} = \pi_{\mathbf{y}}$ and $(G_\theta)_\# \pi_{\mathbf{y}} = \pi_{\mathbf{x}}$ through a GAN objective. This model is schematically represented in Figure \ref{fig:cyclegan}.
As noticed in \cite{cyclegan_zhu2020unpaired}, cycle-consistency in cycle-GAN architectures can be seen as a way to regularize the optimal pair of generators  $H_\sigma$, $G_\theta$ by enforcing the validity of a transitivity property. This allows, in the training phase, to reduce the pairs of generators such that $(H_\sigma)_\# \pi_{\mathbf{x}} = \pi_{\mathbf{y}}$ and $(G_\theta)_\# \pi_{\mathbf{y}} = \pi_{\mathbf{x}}$, favoring a faster and more stable training.
\begin{figure}
    \centering
\begin{tikzpicture}[scale=1]
\draw [draw=black] (0,0) rectangle (2,1) node[pos=.5] {$\mathbb{X}$};
\draw [-stealth](1,1) -- (1,1.5) node[above]{$d^\mathbb{X}_{\tilde \theta}$}; 
\draw [-stealth](2,0.8) -- (4,0.8) node[pos=.5, above] {$H_\sigma$};
\draw [draw=black] (4,1) rectangle (6,0) node[pos=.5] {$\mathbb{Y}$}; 
\draw [-stealth](4,0.2) -- (2,0.2)node[pos=.5, below] {$G_\theta$};
\draw [-stealth](5,1) -- (5,1.5) node[above]{$d^\mathbb{Y}_{\tilde \sigma}$}; 
\end{tikzpicture}
\caption{\centering Schematic representation of a cycle-GAN model}\label{fig:cyclegan}
\end{figure}
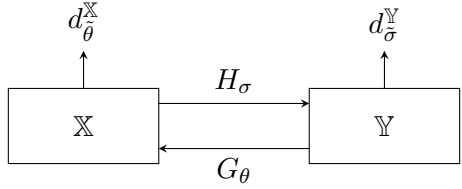

The objective of cycle-GAN is given by the sum of two GAN losses together with the cycle-consistency loss:
\begin{align}\label{eq:cgan}
    \min_{\theta,\sigma}\, \max_{\tilde \theta, \tilde \sigma}\, \alpha\,\mathcal{L}^\mathbb{X}_{GAN}(G_\theta, d^{\mathbb{X}}_{\tilde \theta}) + \beta  \mathcal{L}^\mathbb{Y}_{GAN}(H_\sigma, d^{\mathbb{Y}}_{\tilde \sigma}) + \mathcal{L}_{cycle}(H_{\sigma}, G_{\sigma}),
\end{align}
where $\alpha$ and $\beta$ are positive parameters and 
\begin{align}
    \mathcal{L}^\mathbb{X}_{GAN}(G_\theta, d^{\mathbb{X}}_{\tilde \theta}) &= \int_\mathbb{Y} \log(1 - d^{\mathbb{X}}_{\tilde \theta} (G_\theta(y))\, \mathrm{d}\pi_{\mathbf{y}} + \int_\mathbb{Y} \log(d^{\mathbb{X}}_{\tilde \sigma}(y))\, \mathrm{d}\pi_{\mathbf{y}}, \nonumber\\
    \mathcal{L}^\mathbb{Y}_{GAN}(H_\sigma, d^{\mathbb{Y}}_{\tilde \sigma}) &= \int_\mathbb{X} \log(1 - d^{\mathbb{Y}}_{\tilde \sigma} (H_\sigma(x))\, \mathrm{d}\pi_{\mathbf{x}} + \int_\mathbb{X} \log(d^{\mathbb{Y}}_{\tilde \theta}(x))\, \mathrm{d}\pi_{\mathbf{x}},\,\,\,\text{and}\nonumber\\
\mathcal{L}_{cycle}(H_\sigma, G_{\theta})  &= \int_\mathbb{X} \|G_{\theta}(H_\sigma(x)) - x\|_1\,  \mathrm{d}\pi_{\mathbf{x}} + \int_\mathbb{Y} \|H_{\sigma}(G_\theta(y)) - y\|_1\,  \mathrm{d}\pi_{\mathbf{y}} \label{eq:cycle}.
\end{align}
The training is performed by optimizing \eqref{eq:cgan} computed on the empirical approximations $\pi_{\mathbf{x}} \sim \frac{1}{N} \sum_{i=1}^N \delta_{x^{(i)}}$ and $\pi_{\mathbf{y}} \sim \frac{1}{M} \sum_{i=1}^M \delta_{y^{(i)}}$, where $(x^{(i)})_{i=1}^{N_1},(y^{(i)})_{i=1}^{N_2}$ are training samples from $\mathbb{X}$ and $\mathbb{Y}$ respectively. It is important to note here that the method is unsupervised since the training samples are unpaired, i.e., $y^{(i)}$ does not necessarily correspond to the noisy measurement of $x^{(i)}$. This allows for more flexible models that do not require balanced samples. Moreover, it offers methods able to address more realistic real-world applications, since it is generally difficult and expensive to acquire paired samples.

Despite the original cycle-GAN approach being designed mainly for image-to-image translation, several of its variants have been proposed to address different tasks in an unsupervised framework, such as CT-reconstruction \cite{kang2019cycle}, super-resolution \cite{yuan2018unsupervised}, and conditional image generation \cite{lu2018attribute}, to name a few. However, the successful application of cycle-GAN-based models to inverse problems has remained problematic, primarily due to the following reasons:
\begin{enumerate}
    \item It is unclear how to introduce the knowledge of the forward operator into the model.
    \item Cycle-GAN is a symmetric architecture and struggles to take into account the potential difference in complexity between data $\mathbf{x}$ and measurements  $\mathbf{y}$.
\end{enumerate}
\subsubsection{Optimal transport and cycle-consistency combined}\label{sec:OTcycle}
To address the difficulties stated above in an unsupervised setting, new models based on optimal transport methods have been proposed in \cite{cyclegan_unsup_jong, uar_neurips2021}. In \cite{cyclegan_unsup_jong}, a cycle-GAN architecture was adapted to the $1$-Wasserstein loss by coupling two generators $G_\theta$ and $H_\sigma$, trained as in WGAN to minimize 
\begin{align}
    W_1((H_\sigma)_\# \pi_{\mathbf{x}} , \pi_{\mathbf{y}}) \quad  \text{and} \quad W_1((G_\theta)_\# \pi_{\mathbf{y}},,\pi_{\mathbf{x}})
\end{align}
together with a cycle-consistency loss (c.f. Figure \ref{fig:wgancycle}). This leads to the training objective
\begin{figure}[t!]
    \centering
\begin{tikzpicture}[scale=1]
\draw [draw=black] (0,0) rectangle (2,1) node[pos=.5] {$\mathbb{X}$};
\draw [-stealth](1,1) -- (1,1.5) node[above]{$g^\mathbb{X}_{\tilde \theta}$}; 
\draw [-stealth](2,0.8) -- (4,0.8) node[pos=.5, above] {$H_\sigma$};
\draw [draw=black] (4,1) rectangle (6,0) node[pos=.5] {$\mathbb{Y}$}; 
\draw [-stealth](4,0.2) -- (2,0.2)node[pos=.5, below] {$G_\theta$};
\draw [-stealth](5,1) -- (5,1.5) node[above]{$g^{\mathbb{Y}}_{\tilde \sigma}$}; 
\end{tikzpicture}
\caption{\centering Schematic representation of a cycle-WGAN model}\label{fig:wgancycle}
\end{figure}
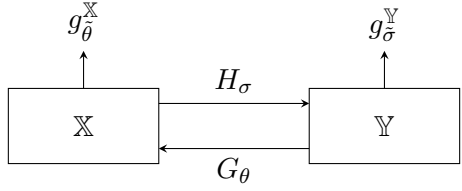
\begin{align}\label{eq:obb}
    \min_{\theta,\sigma}\, \max_{\tilde \theta, \tilde \sigma} \, \alpha\,\mathcal{L}^{\mathbb{X}}_{W}(G_\theta, g^{\mathbb{X}}_{\tilde \theta}) +  \beta\mathcal{L}^{\mathbb{Y}}_{W}(H_\sigma, g^{\mathbb{Y}}_{\tilde \sigma}) + \mathcal{L}_{cycle}(H_{\sigma}, G_{\theta}),
\end{align}
where $\alpha,\beta$ are positive parameters, with
\begin{align*}
    \mathcal{L}^{\mathbb{X}}_{W}(G_\theta, g^{\mathbb{X}}_{\tilde \theta}) = \int_\mathbb{Y} g^{\mathbb{X}}_{\tilde \theta}(G_\theta(y))\, \mathrm{d}\pi_{\mathbf{y}} - \int_\mathbb{X} g^{\mathbb{X}}_{\tilde \theta}(x)\, \mathrm{d} \pi_{\mathbf{x}} + \lambda \int_{\mathbb{X}} (|\nabla g^{\mathbb{X}}_{\tilde \theta}|(\hat x) - 1)^2_+\, \mathrm{d}\hat \pi^{\mathbb{X}},\\
   \mathcal{L}^{\mathbb{X}}_{W}(H_\sigma, g^{\mathbb{Y}}_{\tilde \sigma}) = \int_\mathbb{Y} g^{\mathbb{Y}}_{\tilde \sigma}(H_\sigma(x))\, \mathrm{d}\pi_{\mathbf{x}} - \int_\mathbb{X} g^{\mathbb{Y}}_{\tilde \sigma}(y)\, \mathrm{d} \pi_{\mathbf{y}} + \lambda \int_{\mathbb{Y}} (|\nabla g^{\mathbb{Y}}_{\tilde \sigma}|(\hat y) - 1)^2_+\, \mathrm{d}\hat \pi^{\mathbb{Y}},
\end{align*}
where $\lambda >0$, $\hat \pi^{\mathbb{X}}$ and $\hat \pi^\mathbb{Y}$ are defined as in \eqref{eq:obwgan},   and $\mathcal{L}_{cycle}(G_{\theta}, H_{\sigma})$ is as in \eqref{eq:cycle}. 

It is important to note that the training objective in \eqref{eq:obb} is symmetric in $\mathbb{X}$ and $\mathbb{Y}$, and it is not designed to capture a statistical relationship between $\stx$ and $\sty$.
In the works of \cite{cyclegan_unsup_jong} and \cite{uar_neurips2021}, \eqref{eq:obb} has been accordingly modified to include the knowledge of the inverse problem data acquisition process $\sty = A \stx + \stw$, where $A$ is the measurement operator defined in \eqref{eq:measurement_eq}, and $\sty$ and $\stx$ are the random variables representing ground-truth and noisy measurements.
To this end, \cite{cyclegan_unsup_jong, uar_neurips2021} adapt \eqref{eq:obb} by fixing one of the two generators $H_\sigma$ and $G_\theta$ to be either $A$ or its pseudo-inverse $A^\dagger$ (see also Figure \ref{fig:wgancycle}). At the cost of limiting the expressivity of the cycle architecture, this choice introduces the data acquisition process in the model leading to great benefits in the form of higher stability in the training phase and better data consistency.
Alternatively, it is also possible to assume additional structure on the measurement operators, without fixing it, for example prescribing that the measurement is an unknown convolutional operator of the type $\mathcal{H}_\sigma(x) = h_\sigma \star x$ for a parameterized 
family of convolutional kernels $h_\sigma$ (c.f. Figure \ref{fig:cycletype}). 
\begin{figure}[h!]
\begin{minipage}[c]{0.49\textwidth}
\centering
\begin{tikzpicture}[scale=1]
\draw [draw=black] (0,0) rectangle (2,1) node[pos=.5] {$\mathbb{X}$};
\draw [-stealth](1,1) -- (1,1.5) node[above]{$g^{\mathbb{X}}_{\tilde \theta}$}; 
\draw [-stealth](2,0.8) -- (4,0.8) node[pos=.5, above] {$A$};
\draw [draw=black] (4,1) rectangle (6,0) node[pos=.5] {$\mathbb{Y}$}; 
\draw [-stealth](4,0.2) -- (2,0.2)node[pos=.5, below] {$G_\theta$};
\end{tikzpicture}
\end{minipage} \ \ 
\begin{minipage}[c]{0.49\textwidth}
\centering
\begin{tikzpicture}[scale=1.1]
\draw [draw=black] (0,0) rectangle (2,1) node[pos=.5] {$\mathbb{X}$};
\draw [-stealth](2,0.8) -- (4,0.8) node[pos=.5, above] {$H_\sigma$};
\draw [draw=black] (4,1) rectangle (6,0) node[pos=.5] {$\mathbb{Y}$}; 
\draw [-stealth](4,0.2) -- (2,0.2)node[pos=.5, below] {$A^\dagger$};
\draw [-stealth](5,1) -- (5,1.5) node[above]{$g^\mathbb{Y}_{\tilde \sigma}$}; 
\end{tikzpicture}
\end{minipage}

\centering
\begin{tikzpicture}[scale=1.1]
\draw [draw=black] (0,0) rectangle (2,1) node[pos=.5] {$\mathbb{X}$};
\draw [-stealth](1,1) -- (1,1.5) node[above]{$g^{\mathbb{X}}_{\tilde \theta}$}; 
\draw [-stealth](2,0.8) -- (4,0.8) node[pos=.5, above] {$\mathcal{H}_\sigma$};
\draw [draw=black] (4,1) rectangle (6,0) node[pos=.5] {$\mathbb{Y}$}; 
\draw [-stealth](4,0.2) -- (2,0.2)node[pos=.5, below] {$G_\theta$};
\draw [-stealth](5,1) -- (5,1.5) node[above]{$g^{\mathbb{Y}}_{\tilde \sigma}$}; 
\end{tikzpicture}
\caption{Schematic representation of WGAN-cycle--type models. On the top left: the generator $H_\sigma: \mathbb{X} \rightarrow \mathbb{Y}$ is chosen to be the measurement operator $A$. On the top right: the generator $G_\theta: \mathbb{Y} \rightarrow \mathbb{X}$ is chosen to be the pseudo-inverse $A^\dagger$. On the bottom: the generator $\mathcal{H}_\sigma : \mathbb{X} \rightarrow \mathbb{Y}$ is parametrized by a convolutional operator $\mathcal{H}_\sigma(x) = h_\sigma \star x$.}\label{fig:cycletype}
\end{figure}
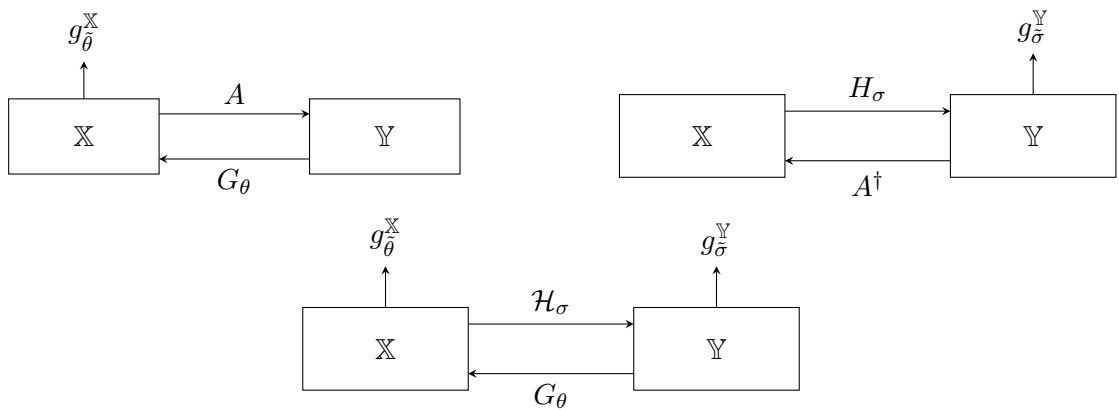

\noindent The objective \eqref{eq:obb} can be adapted in several ways depending on how the data acquisition process has been incorporated. For instance, when only the measurement operator $A$ is prescribed, given suitable losses $f_1 : \mathbb{X} \times \mathbb{X} \rightarrow \R_+$, $f_2 : \mathbb{Y} \times \mathbb{Y} \rightarrow \R_+$, the cycle-loss can be written as 
\begin{align}\label{eq:lo1}
\mathcal{L}_{\rm cycle}(G_{\theta})  = \int_\mathbb{X} f_1(G_{\theta}(A(x)) , x)\,  \mathrm{d}\pi_{\mathbf{x}} + \int_\mathbb{Y} f_2(A(G_\theta(y)) , y)\,  \mathrm{d}\pi_{\sty},
\end{align}
as in \cite{cyclegan_unsup_jong}, or alternatively as
\begin{align}\label{eq:lo2}
\widetilde{\mathcal{L}}_{\rm cycle}(G_{\theta})  = \int_\mathbb{X} f_1(G_{\theta}(A(x)) , x)\,  \mathrm{d}\pi_{\mathbf{x}}, \ \ \overline{\mathcal{L}}_{\rm cycle}(G_{\theta})  = \int_\mathbb{Y} f_2(A(G_\theta(y)) , y)\,  \mathrm{d}\pi_{\sty},
    \end{align}
as in \cite{uar_neurips2021}. In particular, the choice of $\widetilde{\mathcal{L}}_{\rm cycle}$ leads to the Unrolled Adversarial Regularizer (UAR) introduced in \cite{uar_neurips2021}. All these models are trained by computing the objective on the empirical approximation $\pi_{\mathbf{x}} \sim \frac{1}{N_1} \sum_{i=1}^{N_1} \delta_{x^{(i)}}$ and $\pi_{\sty} \sim \frac{1}{N_2} \sum_{i=1}^{N_2} \delta_{y^{(i)}}$ where  $(x^{(i)})_{i=1}^{N_1},(y^{(i)})_{i=1}^{N_2}$ are training samples from $\mathbb{X}$ and $\mathbb{Y}$. It is important to note here that the training samples are unpaired, i.e., they are sampled from the marginal distributions of the ground-truth images and the data, and not from their joint distribution. This is a striking difference compared to standard supervised methods such as U-net post-processing \cite{postprocessing_cnn} and the learned primal-dual (LPD) method \cite{lpd_tmi}. 

\begin{figure}[h!]
\centering
\begin{subfigure}{.32\textwidth}
\includegraphics[width=1.85in]{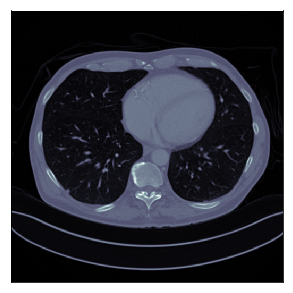}
\centering\vskip-0.3\baselineskip
  {\small Ground-truth}
\end{subfigure}
\begin{subfigure}{.32\textwidth}
	\includegraphics[height=1.85in]{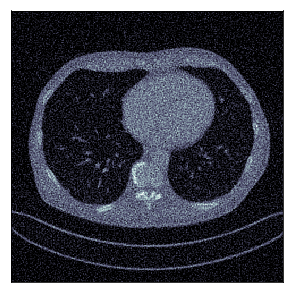}
 \centering\vskip-0.3\baselineskip
  {\small FBP: 21.59, 0.24}
 \end{subfigure}
 \begin{subfigure}{.32\textwidth}
	\includegraphics[width=1.85in]{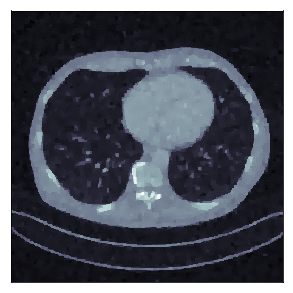}
 \centering\vskip-0.3\baselineskip
  {\small TV: 29.16, 0.77}
 \end{subfigure}
 \\
	\begin{subfigure}	{.32\textwidth}	\includegraphics[width=1.85in]{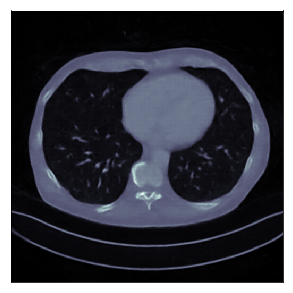}
  \centering\vskip-0.3\baselineskip
  {\small U-net: 32.69, 0.87}
	\end{subfigure}
	\begin{subfigure}{.32\textwidth}
 \includegraphics[width=1.85in]{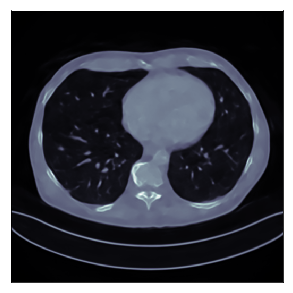}
 \centering\vskip-0.3\baselineskip
  {\small LPD: 34.05, 0.89}
	\end{subfigure}
		\begin{subfigure}{.32\textwidth}
  \includegraphics[width=1.85in]{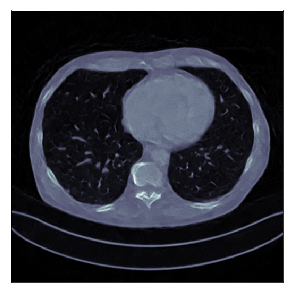}
  \centering\vskip-0.3\baselineskip
  {\small UAR: 32.80, 0.86}
		\end{subfigure}		
  \caption{CT reconstructions on Mayo Clinic data using model-based (FBP, TV), supervised (U-net, LPD), and unsupervised (UAR) methods. The PSNR and SSIM metrics are reported for each reconstruction.}
  \label{fig:uarexp}
\end{figure}
In Figure \ref{fig:uarexp}, we show the experimental results obtained in \cite{uar_neurips2021}, where UAR is applied to produce tomographic reconstructions on the Mayo Clinic low-dose CT grand challenge data-set of abdominal CT scans \cite{mayo_ct_challenge}, whose sinograms are corrupted by Gaussian noise.
We compare UAR to model-based approaches such as the classical filtered back-projection (FBP) and total variation (TV) regularization. Additionally, we choose LPD \cite{lpd_tmi} and U-net post-processing \cite{postprocessing_cnn} as representative supervised methods for inverse problems.

Different choices of the cycle-consistency loss enforce different transitivity properties on the pair $(G_\theta, H_\sigma)$ affecting the reconstruction. For example, the cycle-loss \eqref{eq:lo1} imposes a much stronger constraint on the reconstruction compared to \eqref{eq:lo2}, potentially undermining the expressive power of high-dimensional neural networks. Moreover, the choice of the parameters $\alpha$ and $\beta$ in \eqref{eq:obb} that regulate the strength of the cycle-loss penalization have a strong impact on the reconstruction. This has been analyzed in \cite{uar_neurips2021} for the case where $H_\sigma = A$ and $\ell(z) = \|z\|_2^2$, showing that when $\alpha$ is small, then the reconstruction is very realistic in the sense that $W_1(\pi_{\mathbf{x}},(G_\theta)_{\#}\pi_{\sty}) \approx 0$, but the measurement-consistency cannot be ensured. Similarly, when $\alpha$ is large, even if cycle consistency is ensured, the reconstruction is not guaranteed to lie in the data manifold (see Figure \ref{eq:differentalpha}).

\begin{figure*}[h!]
	\centering
	\begin{subfigure}{.23\textwidth}
 \includegraphics[height=1.45in]{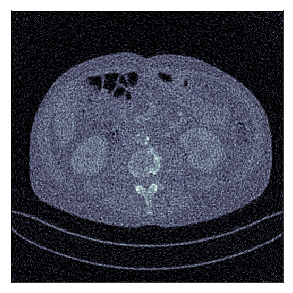}
 \centering\vskip-0.3\baselineskip
  {\scriptsize $\alpha$=0.001: 21.60, 0.21}
	\end{subfigure}	    
\begin{subfigure}{.23\textwidth}
\includegraphics[width=1.45in]{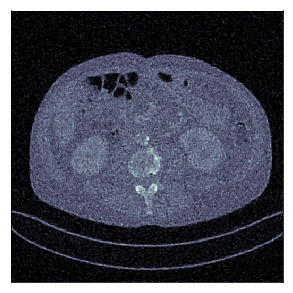}
\centering\vskip-0.3\baselineskip
  {\scriptsize $\alpha$=0.01: 25.33, 0.37}
    \end{subfigure}
\begin{subfigure}{.23\textwidth}
\includegraphics[height=1.45in]{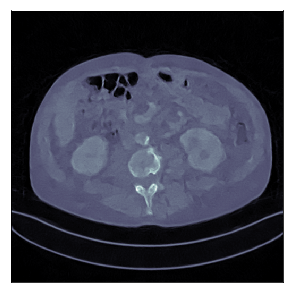}
\centering\vskip-0.3\baselineskip
  {\scriptsize $\alpha$=0.1: 34.65, 0.88}
\end{subfigure}
	\begin{subfigure}{.23\textwidth}
 \includegraphics[width=1.45in]{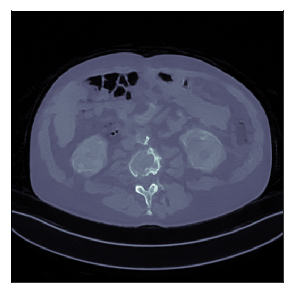}
 \centering\vskip-0.3\baselineskip
  {\scriptsize $\alpha$=1.0: 33.96, 0.88}
	\end{subfigure}
	\caption{Reconstruction of UAR for different $\alpha$. For $\alpha\rightarrow 0$, the unrolled generator (reconstruction operator) seeks to find the minimizer of the expected data-fidelity loss, hence the reconstruction looks similar to FBP.} \label{eq:differentalpha}
 \end{figure*}
\noindent These observations have been formalized in \cite{uar_neurips2021} in the form of the following theorem.
\begin{theorem}
Under suitable assumptions on $\theta$ and $G_\theta$ (see \cite[Section 3]{uar_neurips2021} for more details) the following statements hold:
 \begin{enumerate} 
        \item As $\alpha\rightarrow 0$, $G_\theta\rightarrow G_{\theta_1^*}$ (up to subsequences), where 
        \begin{align}
        \theta_1^*\in \underset{\theta: {\Large\int_{\mathbb{Y}} } \left\|y - AG_\theta(y)\right\|^2_2 \, \mathrm{d}\pi_{\sty} = 0}{\argmin}\,\, W_1(\pi_{\mathbf{x}},(G_\theta)_{\#}\pi_{\mathbf{y}}).
        \end{align}
        \item As $\alpha\rightarrow \infty$, $G_\theta\rightarrow G_{\theta_2^*}$ (up to subsequences), where 
        \begin{align}
        \theta_2^* \in \underset{\theta: (G_\theta)_\# \pi_{y^\delta} = \pi_{\mathbf{x}}}{\argmin}\,\, \int_{\mathbb{Y}} \left\|y-AG_\theta(y)\right\|^2_2\, \mathrm{d}\pi_{\sty}.
        \end{align}
\end{enumerate}
\end{theorem}

A cycle-GAN-style approach for unsupervised learning of unrolled operators referred to as the \textit{adversarially learned primal-dual} (ALPD), was introduced in \cite{adv_lpd_ssvm2021} and was analyzed under the lenses of variational inference. In particular, in \cite{adv_lpd_ssvm2021} the following objective was considered
\begin{align}
\min_{\theta}\,\operatorname{KL}((G_{\theta})_{\#}\pi_{\sty},\pi_{\mathbf{x}}) + C_1\int_{\mathbb{Y}} & \|y-A(G_{\theta}(y))\|_2^2 \, \mathrm{d}\pi_{\sty}\nonumber \\
& + C_2\int_{\mathbb{X}}\|x-G_{\theta}(A(x))\|_2^2\, \mathrm{d}\pi_{\mathbf{x}}.
    \label{overall_ml_final}
\end{align}
It was demonstrated in \cite{adv_lpd_ssvm2021} that under appropriately defined statistical models for $\pi_\mathbf{x}$ and $\pi_{\sty}$ and suitably chosen constants $C_1$ and $C_2$, the maximum likelihood estimate of the parameter $\theta$ leads to the training objective in \eqref{overall_ml_final}. Moreover, replacing the ${\rm KL}$ divergence term with the 1-Wasserstein distance leads to a training loss that is identical to the one proposed in \cite{cyclegan_unsup_jong} in the special case where the forward operator is known. The reconstructed images using the trained model $G_{\theta}$ are shown in Figures \ref{ct_image_figure_shepplogan} and \ref{ct_image_figure_mayo2} for the Shepp-Logan phantom and the low-dose Mayo CT images, respectively. Both experiments reveal that an adversarially trained unrolled operator as proposed in \cite{adv_lpd_ssvm2021} does a better job of preserving the image textures better than unrolled operators trained in a supervised manner using the standard squared error loss. This behavior is consistent with the fact that supervised approaches trained by minimizing the $\ell_2^2$ error effectively produce an approximation to the posterior mean of the image conditioned on the data, which is inherently an averaging operator, unlike a likelihood maximization approach.  
\begin{figure*}[t]
\begin{minipage}[t]{\textwidth}
  \centering
  \vspace{0pt}  
  \includegraphics[width=0.28\textwidth]{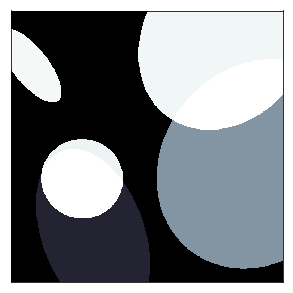}
  \qquad
  \includegraphics[width=0.28\textwidth,trim=30 27 0 0,clip,angle=90]{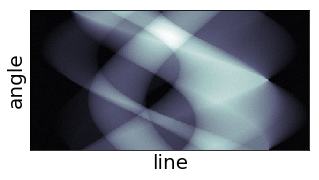}
  \vskip-0.1\baselineskip
  {\small Example of training data: image and its corresponding projection data (sinogram)}
\end{minipage}	
\\[1em] 
\begin{minipage}[t]{0.24\textwidth}
  \centering
  \vspace{0pt}  
  \includegraphics[width=\linewidth]{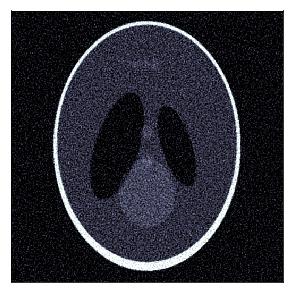}
  \vskip-0.5\baselineskip
  {\scriptsize FBP: 19.51 dB, 0.13}
\end{minipage}	
\begin{minipage}[t]{0.24\textwidth}
  \centering
  \vspace{0pt}  
  \includegraphics[width=\linewidth]{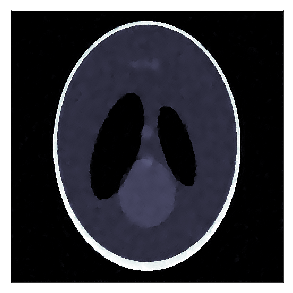}
  \vskip-0.5\baselineskip
  {\scriptsize TV: 29.18 dB, 0.84}
\end{minipage}	
\begin{minipage}[t]{0.24\textwidth}
  \centering
  \vspace{0pt}  
  \includegraphics[width=\linewidth]{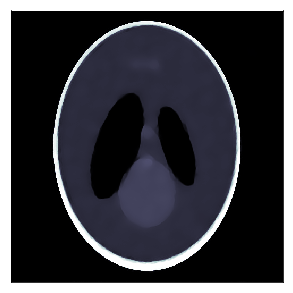}
  \vskip-0.5\baselineskip
  {\scriptsize LPD: 27.89 dB, 0.96}
\end{minipage}	
\begin{minipage}[t]{0.24\textwidth}
  \centering
  \vspace{0pt}  
  \includegraphics[width=\linewidth]{ALPD_cycleGAN_images/adv_lpd_5x5_filters_15layers_shepplogan_001}
  \vskip-0.5\baselineskip
  {\scriptsize ALPD: 28.27 dB, 0.90}
\end{minipage}	
\caption{\small{Comparison of supervised and unsupervised training on the Shepp-Logan phantom. The PSNR (dB) and SSIM are indicated below the images. ALPD does a better job of alleviating over-smoothing, unlike its supervised variant (LPD).}}
\label{ct_image_figure_shepplogan}
\end{figure*}
\begin{figure*}[t]
\begin{minipage}[t]{0.32\textwidth}
  \centering
  \vspace{0pt}
  \begin{tikzpicture}[spy using outlines={circle,red,magnification=4.0,size=1.50cm, connect spies}]   
    \node {\includegraphics[width=\linewidth]{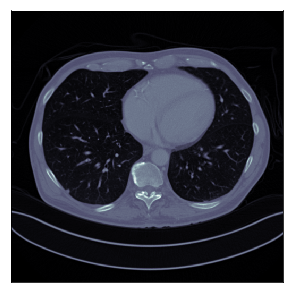}};
    \spy on (0.02,-0.68) in node [left] at (1.9,1.25);
    \spy on (-0.54,-0.76) in node [left] at (-0.4,1.25);  
  \end{tikzpicture}
  \vskip-0.5\baselineskip
  {\small Ground-truth}
\end{minipage}
\begin{minipage}[t]{0.32\textwidth}
  \centering
  \vspace{0pt}
  \begin{tikzpicture}[spy using outlines={circle,red,magnification=4.0,size=1.50cm, connect spies}]   
    \node {\includegraphics[width=\linewidth]{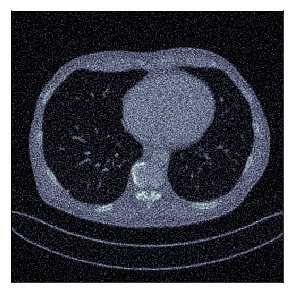}};    
    \spy on (0.02,-0.68) in node [left] at (1.9,1.25);
    \spy on (-0.54,-0.76) in node [left] at (-0.4,1.25);  
  \end{tikzpicture}
  \vskip-0.5\baselineskip  
  {\small FBP: 21.63 dB, 0.24}
\end{minipage}
\begin{minipage}[t]{0.32\textwidth}
  \centering
  \vspace{0pt}
  \begin{tikzpicture}[spy using outlines={circle,red,magnification=4.0,size=1.50cm, connect spies}]   
    \node {\includegraphics[width=\linewidth]{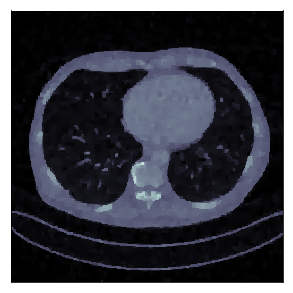}};
    \spy on (0.02,-0.68) in node [left] at (1.9,1.25);
    \spy on (-0.54,-0.76) in node [left] at (-0.4,1.25);  
  \end{tikzpicture}
  \vskip-0.5\baselineskip  
  {\small TV: 29.25 dB, 0.79}
\end{minipage}
\\[1em] 
\begin{minipage}[t]{0.32\textwidth}
  \centering
  \vspace{0pt}
  \begin{tikzpicture}[spy using outlines={circle,red,magnification=4.0,size=1.50cm, connect spies}]   
    \node {\includegraphics[width=\linewidth]{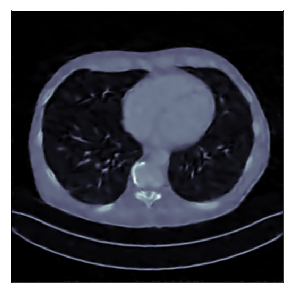}};
    \spy on (0.02,-0.68) in node [left] at (1.9,1.25);
    \spy on (-0.54,-0.76) in node [left] at (-0.4,1.25);  
  \end{tikzpicture}
  \vskip-0.5\baselineskip
  {\small AR: 31.83 dB, 0.84}
\end{minipage}
\begin{minipage}[t]{0.32\textwidth}
  \centering
  \vspace{0pt}
  \begin{tikzpicture}[spy using outlines={circle,red,magnification=4.0,size=1.50cm, connect spies}]   
    \node {\includegraphics[width=\linewidth]{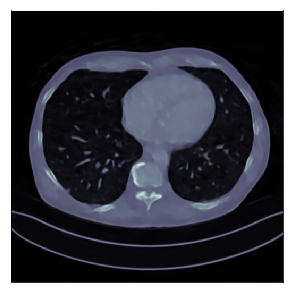}};    
    \spy on (0.02,-0.68) in node [left] at (1.9,1.25);
    \spy on (-0.54,-0.76) in node [left] at (-0.4,1.25);  
  \end{tikzpicture}
  \vskip-0.5\baselineskip  
  {\small LPD: 33.39 dB, 0.88}
\end{minipage}
\begin{minipage}[t]{0.32\textwidth}
  \centering
  \vspace{0pt}
  \begin{tikzpicture}[spy using outlines={circle,red,magnification=4.0,size=1.50cm, connect spies}]   
    \node {\includegraphics[width=\linewidth]{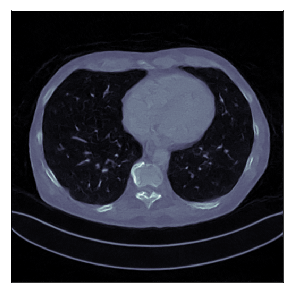}};
    \spy on (0.02,-0.68) in node [left] at (1.9,1.25);
    \spy on (-0.54,-0.76) in node [left] at (-0.4,1.25);  
  \end{tikzpicture}
  \vskip-0.5\baselineskip  
  {\small ALPD: 32.48 dB, 0.84}
\end{minipage}
\caption{\small{Comparison of ALPD with some classical model- and data-driven reconstruction methods on the Mayo Clinic data. The corresponding PSNR (dB) and SSIM are indicated below the images and the key differences in the reconstructed images are highlighted. The ALPD reconstruction is visibly sharper as compared to LPD, enabling easier identification of clinically important features.}}
\label{ct_image_figure_mayo2}
\end{figure*}
\subsection{Adversarial regularization}
Another notable alternative approach to include a learned regularization in the reconstruction process is to first learn an explicit regularization functional in \eqref{eq:var_reg} and to solve the resulting variational problem subsequently. One such option is to learn an \textit{adversarial regularizer}, which was first proposed and analyzed in \cite{ar_nips} and subsequently specialized to adversarial convex regularizers in \cite{acr_arxiv}. Here, the construction of a data-driven regularization is inspired by how discriminative networks (also referred to as \textit{critics}, similarly as in the generative machine learning literature) are trained in the WGAN framework.

To train such an adversarial regularizer, we assume to have $(x^{(i)})_{i=1}^{N_1}\in {\mathbb{X}}$ and $(y^{(j)})_{j=1}^{N_2}\in {\mathbb{Y}}$, which are i.i.d. samples from the marginal distributions $\pi_\stx$ and $\pi_{\sty}$ of ground-truth images and measurement data, respectively. 
Additionally, we assume that there exists a (potentially regularizing) pseudo-inverse $\ForwardOp^\dagger \colon {\mathbb{Y}} \to {\mathbb{X}}$ to the forward operator $\ForwardOp$ and define the measure ${\pi}_\dagger \in \mathcal{P}(\mathbb{X})$ as ${\pi}_\dagger:= \ForwardOp^\dagger_\# (\pi_{\sty})$. Then, the idea of adversarial regularization is to train a regularizer $\RegFunc_{\NNparam}$, parametrized by a neural network, to discriminate between the distributions $\pi_{\mathbf{x}}$ and ${\pi}_\dagger$, i.e. between the distribution of ground-truth images and the distribution of imperfect solutions $\ForwardOp^\dagger y_i$ (i.e., images with noise and artifacts).
More concretely, we compute
\begin{equation}\label{eq:ARGenericTraining}
	R_{\widehat{\sigma}} \colon {\mathbb{X}} \to \Real 
	\quad\text{where}\quad
	\widehat{\sigma} \in \argmin_{\sigma} L(\sigma),
\end{equation}
where $L(\sigma)$ is chosen as
\begin{align}\label{eq:ARob}
    L(\sigma) & =   \int_\mathbb{X} R_\sigma(x)\, {\mathrm d}\pi_{\mathbf{x}}  -  \int_\mathbb{X} R_\sigma(x)\, {\mathrm d}\pi_\dagger  - \lambda \int_\mathbb{X}  \bigr(\|\nabla R_{\sigma}(x)\| - 1 \bigr)_+^2\, {\mathrm d}\hat \pi \nonumber\\
    & =    \int_\mathbb{X} R_\sigma(x)\, {\mathrm d}\pi_{\mathbf{x}} - \int_\mathbb{X} R_\sigma(A^\dagger y)\, {\mathrm d}\pi_{\sty}  - \lambda \int_\mathbb{X}  \bigr(\|\nabla R_{\sigma}(x)\| - 1 \bigr)_+^2\, {\mathrm d}\hat \pi.
\end{align}
Here, $\hat{\pi} \in \mathcal{P}(\mathbb{X})$ is defined by sampling uniformly on the lines connecting samples of $\pi_{\mathbf{x}}$ and samples of $\pi_\dagger$. The heuristic behind this choice is that a regularizer trained this way will penalize noise and artifacts generated by the pseudo-inverse (and contained in ${\pi}_\dagger$). From a theoretical point of view, one can notice that the minimum of  \eqref{eq:ARob} approximates the $1$-Wasserstein distance $W_1(\pi_{\mathbf{x}}, \pi_\dagger)$ between $\pi_{\mathbf{x}}$ and $\pi_\dagger$. Moreover, the optimal $R_{\widehat \sigma}$ approximates the Kantorovich potential for the $1$-Wasserstein distance as defined in Section \ref{sec:dualw}. In particular, the Kantorovich potential for $W_1(\pi_{\mathbf{x}}, \pi_\dagger)$ turns out to be a good regularizer for the given inverse problem.
The resulting regularizer $R_{\widehat{\sigma}}$ is called an adversarial regularizer (AR). In practical applications, the measures $\pi_{\mathbf{x}}, {\pi}_\dagger \in \mathcal{P}(\mathbb{X})$ are replaced with their empirical counterparts given by the training data samples $x_i$ and $\ForwardOp^\dagger y_i$, respectively. Suppose, one computes a gradient step on the learned regularizer, given by $x_\eta=x-\eta\,\nabla_x R_{\widehat{\sigma}}(x)$, starting from $x$ drawn according to $\pi_{\dagger}$. Let ${\pi}^\eta_\dagger$ be the distribution of $x_\eta$. Under appropriate regularity assumptions on the $1$-Wasserstein distance $W_1({\pi}^\eta_\dagger,\pi_\stx)$ (see \cite[Theorem 1]{ar_nips}), one can show that 
\begin{align}\label{eq:wdecay}
    	\frac{\mathrm{d}}{\mathrm{d}\eta} W_1({\pi}^\eta_\dagger,\pi_\stx) |_{\eta=0} = - \int_\mathbb{X} \|\nabla_xR_{\widehat{\sigma}}(x)\|^2\, {\mathrm d} \pi_\dagger.
\end{align}
    This ensures that by taking a small enough gradient step, one can reduce the $1$-Wasserstein distance from the ground-truth $\pi_{\mathbf{x}}$. This is a good indicator that using $R_{\widehat{\sigma}}$ as a variational regularization term and consequently penalizing it implicitly aligns the distribution of regularized solutions with the distribution $\pi_{\mathbf{x}}$ of ground-truth samples. 
Further, one can show that if the AR is Lipschitz-continuous\footnote{1-Lipschitz continuity is approximately enforced by the gradient penalty term in \eqref{eq:ARob}. However, this does not guarantee that the AR is Lipschitz continuous. This property can be instead enforced by choosing the right network architecture. Indeed, all convolutional neural networks with ReLU activations are Lipschitz continuous for some Lipschitz constant $L$, which, albeit, might be arbitrarily large.}, then for a given noisy measurement $y^\delta \in \mathbb{Y}$, a minimizer of the variational problem 
\begin{equation}\label{eq:ARvarcoerc}
	f(y^\delta,A x) + \lambda \left(R_{\widehat{\sigma}}(x) + \epsilon\|x\|_{\mathbb{X}}^2\right),
\end{equation}
exists, where the squared norm on $x$ is needed to enforce coercivity. 
\subsubsection{Adversarial convex regularizer (ACR)}
The adversarial regularizer $R_{\widehat \sigma}$ trained in \eqref{eq:ARGenericTraining} is typically non-convex, due to a typical DNN parameterization of $R_\sigma$. Nevertheless, it is possible to enforce (strong) convexity on $R_\sigma$, leading to the \textit{adversarial convex regularizer} (ACR). The ACR allows for achieving stronger forms of convergence than its non-convex predecessor while precluding discontinuities in the reconstruction operator. This necessitates a suitable parameterization of the learned regularizer. One such option to impose convexity on $R_{\widehat{\NNparam}}$ is to use input convex neural networks \cite{amos2017input}. 
We refer to \cite{acr_arxiv} for more details on the parameterization of ACRs. Given a so-constructed (and adversarially trained) ACR (denoted as $R_{\widehat{\sigma}}$) that is convex in $x$, one then considers a regularization functional of the form
\begin{equation}
	R(x) = R_{\widehat{\sigma}}(x) + \epsilon \left\|x\right\|_{{\mathbb{X}}}^2,
	\label{eq:ACR}
\end{equation}
where $R_{\widehat{\sigma}}:{\mathbb{X}}\rightarrow\mathbb{R}$ is the trained ACR, which we assume to be 1-Lipschitz besides being convex in $x$. The corresponding variational regularization problem then entails minimizing the regularized energy
\begin{equation}
	f(y^\delta,\ForwardOp x)+\lambda R(x),
	\label{eq:var_loss_maindef1}
\end{equation}
with respect to $x\in{\mathbb{X}}$. In this setting, we get the following set of improved theoretical guarantees for the ACR, by following standard arguments in variational calculus.
\begin{theorem}[Properties of Adversarial Convex Regularizers \cite{acr_arxiv}] \label{thm:ACR} \hfill
	\begin{enumerate}
		\item Existence and uniqueness:
		      The functional in \eqref{eq:var_loss_maindef1} is strongly convex in $x$ and has a unique minimizer $\widehat{x}_{\lambda}\left(y\right)$ for every $y\in {\mathbb{Y}}$ and $\lambda>0$. 		      	         
		\item Stability: The optimal solution $\widehat{x}_{\lambda}\left(y\right)$ is continuous in $y$.
  
		\item  Convergence: 
		      For $\delta\rightarrow 0$ and $\lambda(\delta) \rightarrow 0$ such that $\displaystyle\frac{\delta}{\lambda(\delta)}\rightarrow 0$, we have that $\widehat{x}_{\lambda}\left(y^{\delta}\right)$ converges to the $R$-minimizing solution $x^{\dagger}$ given by
        \begin{equation*}
            x^{\dagger} \in \underset{x\in \mathbb{X}}{\argmin}\,R(x)\quad \text{\,\,subject to\,\,}\quad y^0=Ax.
            \label{eq:x_min_sol_acr}
        \end{equation*}
	\end{enumerate}
\end{theorem}
Despite strong theoretical guarantees, the numerical experiments in \cite{acr_arxiv} (especially, for sparse-view CT reconstruction) indicate a lack of expressive power of ACRs as compared to their nonconvex counterpart AR. This underscores the need to develop techniques that achieve a better compromise between empirical performance and theoretical certificates. A step in this direction has been made very recently by relaxing convexity to a so-called convex-nonconvex construction of the regularizer \cite{shumaylov2023provably}, wherein the regularizer is allowed to be nonconvex while still maintaining convexity of the overall variational energy and the classical theoretical guarantees.    

\subsubsection{Combining end-to-end reconstructions and adversarial regularization}\label{sssec:combiningE2EAR}
Cycle-WGAN models such as UAR and adversarial regularizer (AR) are both unsupervised approaches for solving inverse problems while being able to use the knowledge of the measurement operator in the reconstruction process. 
In \cite{uar_neurips2021}, it has been shown that UAR can be combined with AR to improve the quality of the reconstruction. The key observation is that the adversarial regularizer $R_{\widehat \sigma} : \mathbb{X} \rightarrow \R$ is trained to distinguish samples from the noisy reconstruction $(A^\dagger)_{\#} \pi_{\sty}$ from samples from the ground-truth $\pi_{\mathbf{x}}$. Therefore, it is plausible that by substituting $A^\dagger$ with a generator $G_{\widehat \theta}$ learned through UAR, one should be able to improve the noisy reconstruction $(A^\dagger)_{\#} \pi_{\sty}$ using a more accurate reconstruction, given by $(G_{\widehat \theta})_{\#} \pi_{\sty}$ and then construct a regularizer based on it. The noisy reconstruction $(G_{\widehat \theta})_{\#} \pi_{\sty}$ would be an \emph{improved guess} over $(A^\dagger)_{\#} \pi_{\sty}$.
Following this intuition and rewriting the UAR objective \cite{uar_neurips2021} as
\begin{align}\label{eq:UARob}
\min_{\sigma} \max_{\theta}  \,  \int_\mathbb{Y} R_{\sigma}(G_\theta(y))\, \mathrm{d}\pi_{\sty} - & \int_\mathbb{X} R_{\sigma}(x)\, \mathrm{d} \pi_{\mathbf{x}} + \lambda \int_\mathbb{X} (|\nabla R_{\sigma}|(\hat x) - 1)^2_+\, \mathrm{d}\hat \pi \nonumber\\
& + \int_\mathbb{X} f(G_{\theta}(A(x)) , x)\,  \mathrm{d}\pi_{\mathbf{x}},
\end{align}
one observes that 
the optimal $R_{\widehat \sigma}$ is trained to distinguish noisy samples of $(G_{\widehat \theta})_{\#} \pi_{\sty}$ from samples from $\pi_{\mathbf{x}}$ and therefore $R_{\widehat \sigma}$ is a good regularizer for the distribution $(G_{\widehat \theta})_{\#} \pi_{\sty}$. Moreover, in \cite{uar_neurips2021} it has been remarked that since the regularizer $R_{\widehat \sigma}$ is an approximation of the Kantorovich potential for $W_1((G_{\widehat \theta})_{\#} \pi_{\sty}, \pi_{\mathbf{x}})$, it is possible to compute the derivative of the $1$-Wasserstein distance with respect to a GD step as in \eqref{eq:wdecay}. 
Indeed, suppose one considers a GD step of the learned regularizer, given by $x_\eta=x-\eta\,\nabla_xR_{\widehat{\sigma}}(x)$, starting from $x\sim (G_{\widehat \theta})_{\#} \pi_{\sty}$. Let ${\pi}^\eta_{\widehat \sigma}$ be the distribution of $x_\eta$. Under appropriate regularity assumptions on the $1$-Wasserstein distance $W_1({\pi}^\eta_{\widehat \sigma},\pi_{\mathbf{x}})$ (see \cite[Theorem 5]{uar_neurips2021}), one can show that 
\[
	\frac{\mathrm{d}}{\mathrm{d}\eta} W_1({\pi}^\eta_{\widehat \sigma},\pi_{\mathbf{x}}) |_{\eta=0} = - \int_\mathbb{X} \|\nabla_x R_{\widehat{\sigma}}(x)\|^2\, {\mathrm d} {\pi}^\eta_{\widehat \sigma}.
\]
This ensures that by taking a small enough gradient step from samples of $(G_{\widehat \theta})_{\#} \pi_{\sty}$, one can reduce the $1$-Wasserstein distance from the ground-truth $\pi_{\mathbf{x}}$. 
This is a strong theoretical guarantee that $R_{\widehat \sigma}$ is a good regularizer for the regularized inverse problem
\begin{align}\label{eq:regUAR}
    \min_{x\in \mathbb{X}} f(y^\delta ,Ax) + \lambda \left(R_{\widehat{\sigma}}(x) + \epsilon\|x\|_{\mathbb{X}}^2\right),
\end{align}
where $y^\delta \in \mathbb{Y}$ is the noisy measurement. In particular, by taking a few gradient descent steps on the objective in \eqref{eq:regUAR}, initialized with $G_{\widehat \theta}(y^\delta)$, we are moving the end-to-end reconstruction $G_{\widehat \theta}(y^\delta)$ towards the ground-truth distribution $\pi_{\mathbf{x}}$, c.f. Figure \ref{fig:manifold}. This additional \emph{refinement} can be seen as a process of \emph{instance adaptation} of a given end-to-end reconstruction $G_{\widehat \theta}(y^\delta)$.
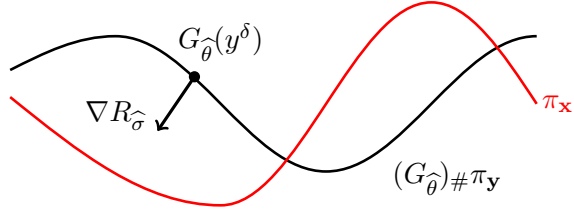
\begin{figure}[t]
\centering
\begin{tikzpicture}[xscale=1.40, yscale=0.9]
  \draw[line width=0.35mm] (0,0.5) sin (1,1) cos (2,0) sin (3,-1) cos (4,0) sin (5,1)
  plot[only marks] coordinates{(1,1)} (2.0,0.4) node [above = 0.7mm]{$ G_{\widehat \theta}(y^\delta)$} (4.15,-1.1) node{$(G_{\widehat \theta})_{\#} \pi_{\sty}$};
   \draw[->,very thick] (1.75,0.40) -- (1.4,-0.4) node[above = 2mm,left]{$\nabla R_{\widehat \sigma}$} ;
  \node at (1.75,0.40)  [circle,fill,inner sep=1.5pt]{};
  \draw[color=red, line width=0.35mm] (0,0.1) sin (2,-1.5) cos (3,0) sin (4,1.5) cos (5,0) (5.2,0) node {$\pi_{\mathbf{x}}$};
 \end{tikzpicture}
 \caption{A schematic illustration of the behavior of the gradient descent step initialized at $G_{\widehat \theta}(y^\delta)$. The gradient descent is moving the point $G_{\widehat \theta}(y^\delta)$  in the direction $\nabla R_{\widehat \sigma}(G_{\widehat \theta}(y^\delta))$. Since $R_{\widehat \sigma}$ is the Kantorovich potential for $W_1((G_{\widehat \theta})_{\#} \pi_{\sty}, \pi_{\mathbf{x}})$ the step of gradient descent moves $G_{\widehat \theta}(y)$ towards the ground-truth distribution $\pi_{\mathbf{x}}$} \label{fig:manifold}
 \end{figure}
In Figure \ref{fig:refined} we report the reconstructions obtained in \cite{uar_neurips2021} using the UAR approach described in Section \ref{sec:OTcycle} (on the left) and the reconstruction obtained by solving \eqref{eq:regUAR} performing few steps of GD initialized at $G_{\widehat \theta}(y^\delta)$ (on the right).

\begin{figure}[h!]
\centering
\begin{subfigure}{.42\textwidth}
\includegraphics[width=2.3in]{figures/uar.png}
\centering\vskip-0.3\baselineskip
  {\small UAR: 32.80 dB, 0.86}
\end{subfigure}\quad \quad
\begin{subfigure}{.42\textwidth}
	\includegraphics[height=2.3in]{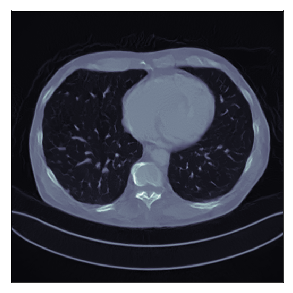}
 \centering\vskip-0.3\baselineskip
  {\small  UAR (refined): 33.15 dB, 0.87}
 \end{subfigure}
 \caption{Comparison between cycle-based end-to-end reconstruction obtained by UAR and the additional refinement step. We report the PSNR (dB) and SSIM below the respective images. The refinement step leads to a minor improvement in the quality of the reconstructed image via \textit{instance adaptation} (i.e., by computing the reconstruction corresponding to a specific realization of the measurement random variable.)}\label{fig:refined}
 \end{figure}

\subsubsection{The refinement step in UAR and Brenier's theorem}\label{sec:brenier}

Adversarial regularizers and the refinement step in UAR can be interpreted under the lenses of Brenier's theorem \cite{brenier1987decomposition} described in Section \ref{sec:brenier}. Indeed, by computing the gradient descent step with respect to the learned regularizer $R_{\widehat \sigma}$ as
\begin{align}\label{eq:descent}
x_\eta=x-\eta\,\nabla_xR_{\widehat{\sigma}}(x)
\end{align}
for either $x = A^\dagger y$ (for AR) or $x = G_{\widehat \theta}(y)$ (for UAR), one is effectively trying to compute an approximation of the optimal transport map from the end-to-end reconstruction to the ground-truth through the formula \eqref{eq:formbrenier}.

Unfortunately, Brenier's theorem holds only for $p$-Wasserstein distances with $1<p<\infty$, and thus it cannot be applied directly to AR and UAR since they are based on an approximation of the $1$-Wasserstein distance. Variants of AR and UAR that use the $p$-Wasserstein distances with $1<p<\infty$ would allow applying Theorem \ref{thm:brenier}, giving the optimal learning step $\eta$ that should be used to compute the optimal transport map. However, such variants would inevitably suffer from the lack of a computationally favorable dual formulation such as the one for the $1$-Wasserstein distance. Notably, \cite{milne2022new} considers obtaining approximations of the $1$-Wasserstein distance through the computed potential to design a better descent step \eqref{eq:descent}. We refer interested readers to \cite{milne2022new} for more details and to \cite{evans1999differential} for a more theoretical discussion about the relation between the $1$-Wasserstein distance and Kantorovich potentials.

\section{Unsupervised approaches rooted in convex analysis and monotone operator theory}\label{sec:unsupervised}
In this section, we give an overview of different unsupervised approaches based on convex analysis and monotone operator theory for solving imaging inverse problems, with special emphasis on the learned optimization-based approaches and the plug-and-play (PnP) denoising framework. 
\subsection{Learned optimization solvers}\label{ssec:l2o}
\textit{Learning-to-optimize} (L2O) is an emerging area at the interface of optimization and machine learning that has recently started to receive popularity various in data science applications, including computational imaging. In this section, we review some recent progress in L2O in the context of computational imaging.

L2O methods learn to efficiently solve a class of optimization problems, by adapting to the structure of the problems and the underlying data distribution. Although L2O has not received strong attention in the imaging community compared to related schemes such as PnP/RED, we believe that it will soon become a major area in imaging, due to the recent rise of computationally intensive learned regularizers. Moreover, for each imaging modality, the imaging system is usually fixed or almost fixed, which is a suitable problem setting to use L2O to develop specialized optimization algorithms for imaging in an application-driven manner. 

The L2O schemes are typically trained in an unsupervised manner, with the goal of accelerating optimization on a class of functions of interest, as outlined in the following. Firstly, the users need to generate a set of training problem instances $\{f_i\}_{i=1}^n$, drawn from the problem class of interest, such as those arising from model-based natural image inpainting or sinogram denoising. One example would be $f_i(x) = \|x-y_i\|^2 + \lambda \|\nabla x\|_1$, corresponding to the TV-based variational model for denoising induced by a noisy image $y_i$. Let us denote the algorithm to be learned as $\mathcal{A}_\theta(f, x_0, N)$, with $\theta$ being the set of trainable parameters within the algorithm. Here $f$ denotes the objective, $x_0$ denotes the initial point of the algorithm, while the third argument $N$ denotes the number of iterations to be executed. The output of the algorithm is denoted as $x_N = \mathcal{A}_\theta(f, x_0, N)$. The unsupervised training objective can typically be written as minimizing the final objective value (averaged over the training problems): 
\begin{equation}
    \theta^\star \in \argmin_{\theta} \frac{1}{n}\sum_{i=1}^n f_i(\mathcal{A}_\theta(f_i, x_0, N)),
\end{equation}
or minimizing the sum of the function values along the optimization path:
\begin{equation}
    \theta^\star \in \argmin_{\theta} \frac{1}{n}\sum_{i=1}^n \sum_{m=1}^N f_i(\mathcal{A}_\theta(f_i, x_0, m)),
\end{equation}
where one seeks to minimize the training problems' objective values as much as possible within $N$ iterations. Due to the computational complexity in training, $N$ cannot be too large. In the context of imaging, the number of unrolling iterations is usually chosen to be on the order of $N=10$. 


In this chapter, we only consider theoretically-principled L2O frameworks which lead to provably convergent algorithms. We will start from the basic scheme of Learned PDHG with trainable step-size parameters \cite{banert2020data}, to more advanced schemes such as the learned mirror descent (LMD) methods \cite{tan2023data}, which are based on trainable mirror maps using input-convex neural networks \cite{amos2017input}.

\subsubsection{Learned algorithmic parameters}
Banert et al. proposed a learned step-size scheme for the class of primal-dual splitting algorithms \cite{banert2020data}, used to solve composite optimization problems of the form:
\begin{equation}
    x^\star \in \argmin_x f(Ax) + R(x).
\end{equation}
In the context of imaging, $f(Ax)$ is a data-fidelity term incorporating a forward operator $A$, while $R(x)$ is a regularization term (such as TV regularization). The step-size selection in the primal-dual splitting scheme has been a challenging problem, since jointly selecting the primal step-size, dual step-size, and the extrapolation parameter is difficult in general and significantly affects the practical performance \cite{goldstein2015adaptive,chambolle2023stochastic}. We present a well-known classical primal-dual splitting method, the primal-dual hybrid gradient (PDHG) algorithm of Chambolle and Pock \cite{chambolle2011first}:
\begin{eqnarray*}
   &&y_{k+1} = \prox_{f^*}^\sigma (y_k + \sigma Av_n),\\
    &&x_{k+1} = \prox_R^\tau(x_k - \tau A^{\top}y_{k+1}),\\
    &&v_{k+1} = x_{k+1} + \theta(x_{k+1} -x_k).
\end{eqnarray*}
One could generalize this splitting by parameterizing the scheme as follows, where $\otimes$ denotes the Kronecker product, and $\diag(A,B)$ represents the diagonal operator with operators $A,B$ on the diagonal: 
\begin{eqnarray*}
    && \left[\begin{matrix}
        w_k \\ y_{k+1}
    \end{matrix}\right] = (A \otimes \Id) \diag(\prox_{f^*}^\sigma, \Id) (B \otimes \Id)\left[\begin{matrix}
        Av_k\\y_k
    \end{matrix} \right]\\
    && \left[\begin{matrix}
        v_k \\ x_{k+1}
    \end{matrix}\right] = (C \otimes \Id) \diag(\prox_{R}^\tau, \Id) (D \otimes \Id)\left[\begin{matrix}
        A^{\top}w_k\\x_k
    \end{matrix} \right],
\end{eqnarray*}
where $A$, $B$, $C$, and $D$ are $2 \times 2$ matrices consisting of learnable parameters \cite{banert2020data}. This formulation includes PDHG as the special case
\begin{equation*}
    A=\left[ \begin{matrix}
        1&0\\1&0
    \end{matrix} \right],B=\left[ \begin{matrix}
        \sigma&0\\0&1
    \end{matrix} \right],C=\left[ \begin{matrix}
        1+\theta&-\theta\\1&0
    \end{matrix} \right],D=\left[ \begin{matrix}
        -\tau&1\\0&1
    \end{matrix} \right].
\end{equation*}
As long as the learned parameters are constrained throughout training within the acceptable range given by the convergence theorems of the primal-dual splitting algorithms, the learned scheme is provably convergent.

\subsubsection{Learned mirror descent with input-convex neural networks}
In the previous section, we presented a basic paradigm for provable L2O, by learning the algorithmic parameters of classical optimizers such as PDHG, while restricting the learnable parameters such that theoretical guarantees hold. Although such schemes can achieve a certain degree of adaptivity and acceleration over classical hand-crafted optimizers while maintaining provable convergence, their potential is limited as they involve few trainable parameters. 

In order to fully utilize the training data and make the algorithm adapt well to the inherent structure of the optimization problem class of interest, we wish to leverage the expressive capacity of deep neural networks within some classical optimizer in a principled manner, ensuring provable convergence. The classical \textit{mirror descent} (MD) algorithm by Yudin and Nemirovski is an ideal candidate for such extension by its nature \cite{nemirovsky1983problem}. Before introducing the MD algorithm, we first define the mirror maps as such.
\begin{definition}[Mirror potentials and mirror maps]
    We define a continuously differentiable and strongly-convex function $\Psi : \mathbb{X} \rightarrow \R$ as a mirror potential, and its gradient $\nabla \Psi: \mathbb{X} \rightarrow (\R^{n})^*$ as the (forward) mirror map \cite{nemirovsky1983problem,tan2023data}.
\end{definition}
Denoting $\Psi^*$ as the convex conjugate of the mirror potential $\Psi$, and the backward mirror map as $\nabla\Psi^* = (\nabla \Psi)^{-1}$, we can write the MD iterates as
\begin{equation}
    x_{k+1} = \nabla \Psi^*[\nabla \Psi(x_k) - t_k \nabla f(x_k)],
\end{equation}
or equivalently,
\begin{equation}
    x_{k+1} = \argmin_{x \in \mathbb{X}} \left\{ \langle x, \nabla f(x_k) \rangle + \frac{1}{t_k}B_{\Psi}(x, x_k)\right\}.
\end{equation}
Here, $B_\Psi(x,y) = \Psi(x) - \Psi(y) - \langle \nabla \Psi(y), x-y \rangle$ denotes the Bregman distance induced by the mirror potential $\Psi$. Observe that for the choice of mirror potential $\Psi(\cdot) = \frac{1}{2}\|\cdot\|_2^2$, we recover gradient descent.

The MD algorithm naturally lends itself to the L2O setting, since we can parameterize the mirror potential using deep neural networks. In particular, by parameterizing $\Psi$ as an input-convex neural network (ICNNs) \cite{amos2017input}, the learned mirror potential is enforced to be a convex function w.r.t. the input, which allows us to inherit the convergence properties of MD. Let the mirror potential $\Psi$ and its conjugate $\Psi^*$ be parameterized by two neural networks $M_\theta$ and $M_\vartheta^*$, respectively, where the condition $M_\vartheta^* \approx (M_\theta)^{-1}$ is enforced through training\footnote{Note that the mirror potential should be strongly-convex to ensure provable convergence. We add a small $\ell_2$ term $\frac{\mu}{2}\|x\|_2^2$ to the usual ICNN parameterization to ensure this.}. We can describe the learned mirror descent (LMD) algorithm as:
\begin{equation}\label{eq:approximateMD}
        \tilde{x}_{k+1} = \nabla M_\vartheta^* (\nabla M_\theta(\tilde{x}_k) - t_k \nabla f(\tilde{x}_k)).
\end{equation}
Due to the inexact inverses, we need to enforce $M_\vartheta^* \approx (M_\theta)^{-1}$ for the convergence of LMD. Hence for this framework, we incorporate an additional regularization in the unsupervised training objective stated previously in this section, where the inexactness $\|\nabla M_\vartheta^* \circ \nabla M_\theta - I\|$ is penalized along the distribution $p_\mathbb{X}$ of the optimized iterates. Denoting the LMD algorithm as $\mathcal{A}_{\theta,\vartheta}$, where $\alpha_m$s are the weights across different iterations, the regularized objective is:
\begin{equation}\label{eq:MDTrainLoss}
    \argmin_{\theta,\vartheta}  \frac{1}{n}\sum_{i=1}^n \sum_{m=1}^N \alpha_m f_i(\mathcal{A}_{\theta,\vartheta}(f_i, x_0, m)) + \mathbb{E}_{x \sim p_\mathbb{X}}[\|(\nabla M_\vartheta^* \circ \nabla M_\theta - I)(x)\|].
\end{equation}
Under standard assumptions in convex optimization, we can provide the following regret bound for LMD which is close to the regret bound for MD, subject to the approximation quality of the $M_\vartheta^* \approx (M_\theta)^{-1}$ encouraged in the training process.

\begin{theorem}[{Regret Bound for LMD \cite{tan2023data}}]\label{thm:approxMD}
Suppose $f$ is $\mu$-strongly convex with parameter $\mu>0$, and $\Psi$ is a mirror potential with strong convexity parameter $\sigma$. Let $\{\tilde{x}_k\}_{k=0}^\infty$ be some sequence in $\mathbb{X} = \R^n$, and $\{x_k\}_{k=1}^\infty$ be the corresponding exact MD iterates evaluated at $\tilde{x}_{k-1}$. We have the following regret-bound:
\begin{equation}
    \begin{split}
        &\sum_{k=1}^K t_k(f(\tilde{x}_k) - f(x^*)) \le \\
        & B(x^*, \tilde{x}_1) + \sum_{k=1}^K \left[\frac{1}{\sigma} t_k^2 \|\nabla f(\tilde{x}_k)\|_*^2 + \left(\frac{1}{2t_k\mu} + \frac{1}{\sigma}\right)\|\nabla M_\theta(\tilde{x}_{k+1}) - \nabla M_\theta(x_{k+1})\|_*^2 \right].
    \end{split}
\end{equation}
\end{theorem}
From this result we can observe that, for the case where $M_\vartheta^* \approx (M_\theta)^{-1}$, the term $\|\nabla M_\theta(\tilde{x}_{k+1}) - \nabla M_\theta(x_{k+1})\|_*^2 \rightarrow 0$ and we recover the standard convergence guarantees for MD. In Figure \ref{fig:denoise_recon} and \ref{fig:denoiseVis}, we demonstrate a numerical example of applying LMD on the total-variation (TV) model-based image denoising task. The LMD and the adaptive LMD (a variant of LMD with learned step-sizes besides the learned mirror maps) were trained with unrolling iteration number $N=10$. We can observe significantly improved convergence rates of LMD over classical solvers which are not data-driven.

To further improve the convergence rates and computational efficiency of LMD, the follow-up work \cite{tan2023boosting} of Tan et al. proposes several extensions utilizing momentum-based acceleration and stochastic gradient approximations. We present one of the extensions with the classical Nesterov-type acceleration technique in optimization, the learned \textit{accelerated} mirror descent (LAMD) algorithm in \Cref{alg:LAMD}.

\begin{algorithm}
\caption{Learned Accelerated Mirror Descent (LAMD) \cite{tan2023boosting}}\label{alg:LAMD}
\begin{algorithmic}[1]
\Require Input $\tilde{x}^{(0)} = \tilde{z}^{(0)} = x^{(0)} \in \mathbb{X}$, parameter $r \ge 3$, step-sizes $t_k$, number of iterations $K$
\State $z^{(0)} = \nabla M_\theta(\tilde{z}^{(0)})$
\For{$k = 0, ..., K$}
    \State $x^{(k+1)} = \lambda_k \nabla M_\vartheta^* ({z}^{(k)}) + (1-\lambda_k) \tilde{x}^{(k)}$ with $\lambda_k = \frac{r}{r+k}$
    \State $z^{(k+1)} = z^{(k)} - \frac{kt_k}{r}\nabla f(x^{(k+1)})$
    \State $\tilde{x}^{(k+1)} = x^{(k+1)} - \gamma t_k \nabla f(x^{(k+1)}) $
\EndFor
\State \textbf{return} $x^{(K+1)}= \lambda_K \nabla M_\vartheta^* ({z}^{(K)}) + (1-\lambda_K) \tilde{x}^{(K)}$
\end{algorithmic}
\end{algorithm}%
\noindent With bounded forward-backward inconsistency, an improved convergence rate of LAMD over vanilla LMD can be established in a way similar to the classical accelerated MD. In Figure \ref{fig:LMD_LAMD}, we present numerical results of LAMD in TV model-based denoising, comparing it to learned solvers such as LMD and LPDHG, as well as the classical optimizers such as gradient descent with Nesterov acceleration. We can observe the superior performance of LAMD in this example.


\begin{figure}[t]
    \centering
\includegraphics[width=0.8\textwidth,height=\textheight,keepaspectratio]{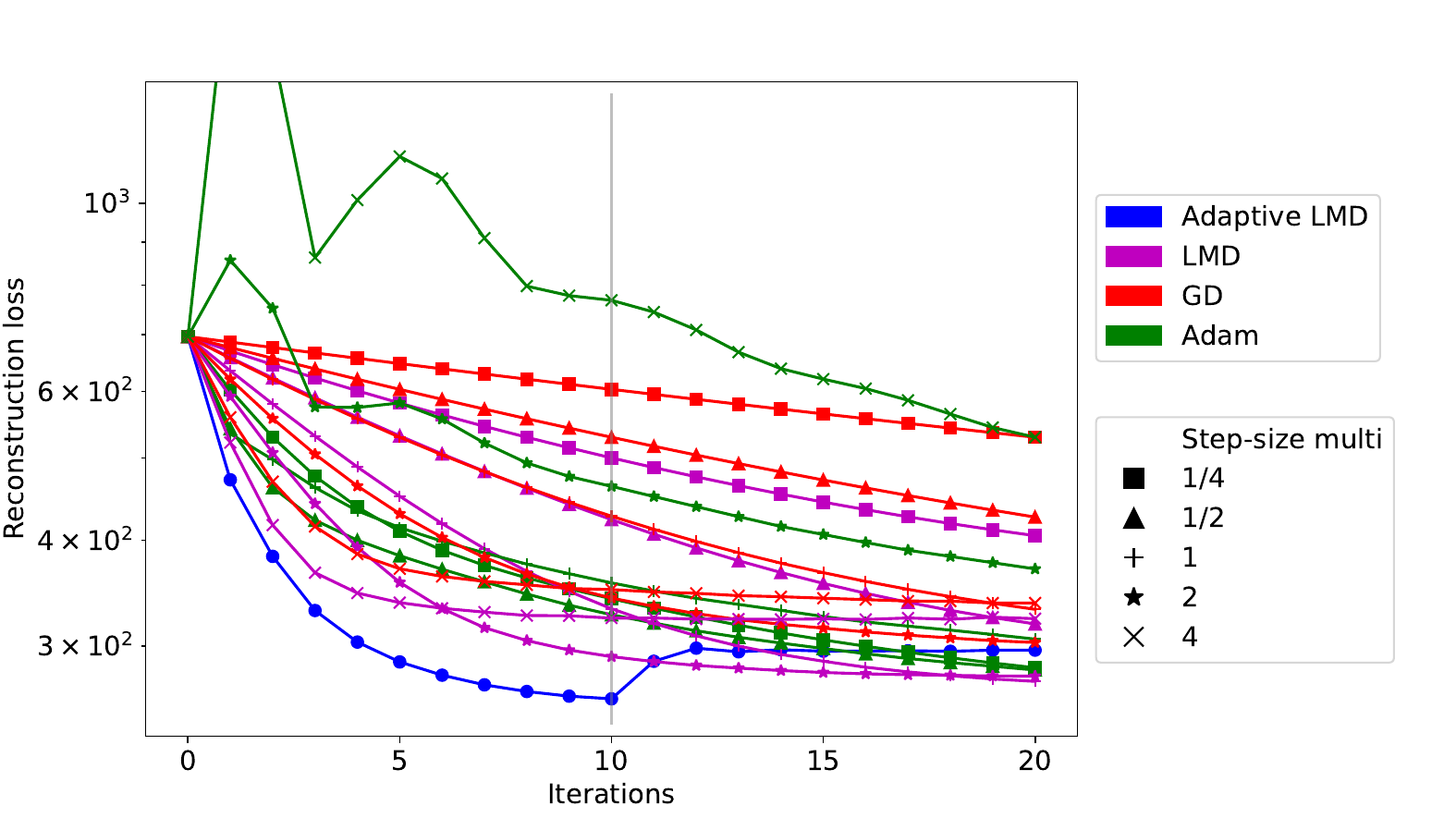}
    \caption{Convergence results of LMD and traditional optimizers in TV denoising problem (see \cite{tan2023data}).}
    \label{fig:denoise_recon}
\end{figure}



\begin{figure}[t]%
\captionsetup[subfloat]{margin=10pt,format=hang,singlelinecheck=false,justification=RaggedRight}

 \centering  \subfloat[\centering Adaptive LMD \newline (3 iterations)]{{\includegraphics[width=0.3\textwidth,height=\textheight,keepaspectratio]{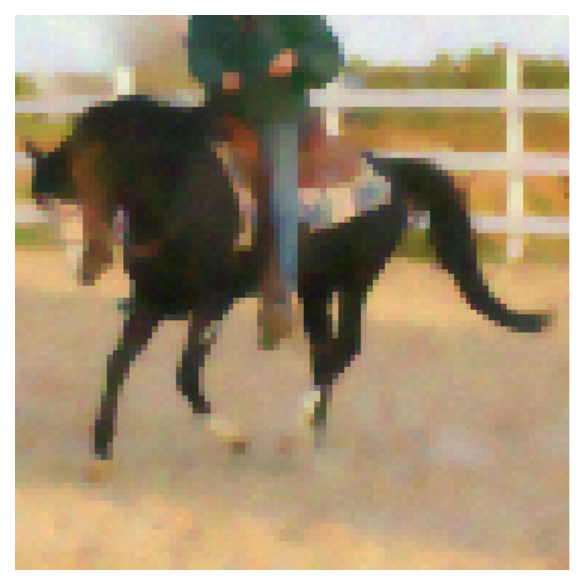}\label{fig:denoiseVismd}}}%
    \subfloat[\centering Adam \newline(3 iterations)]{{\includegraphics[width=0.3\textwidth,height=\textheight,keepaspectratio]{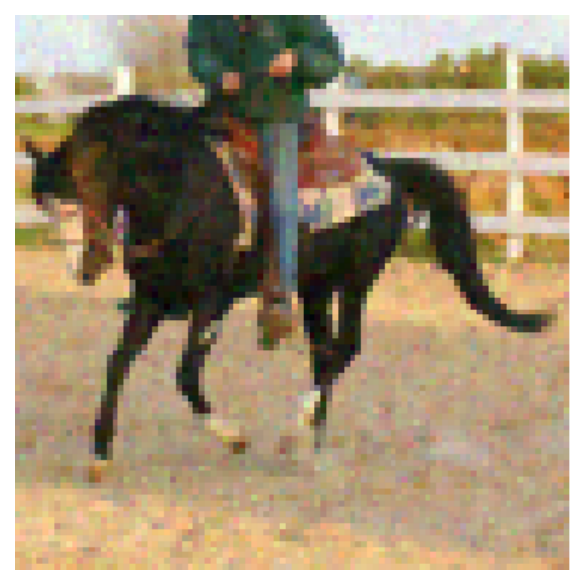}}}%
    \subfloat[\centering Adam \newline(10 iterations)]{{\includegraphics[width=0.3\textwidth,height=\textheight,keepaspectratio]{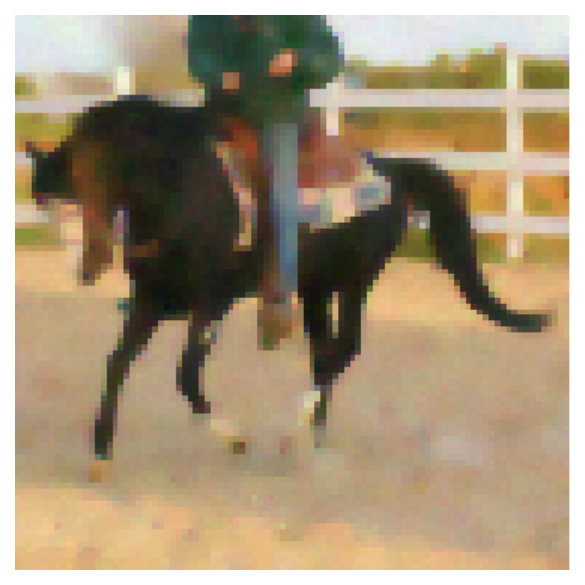}
    \label{fig:denoiseVisadam}}}%
        \caption{Recovered images by LMD and Adam. We can observe visually the LMD achieves much faster convergence over Adam with fewer artifacts in early iterations. See \cite{tan2023data} for further details.}
    \label{fig:denoiseVis}%
\end{figure}


\begin{figure}[h!]
    \centering
    \includegraphics[width=0.7\textwidth]{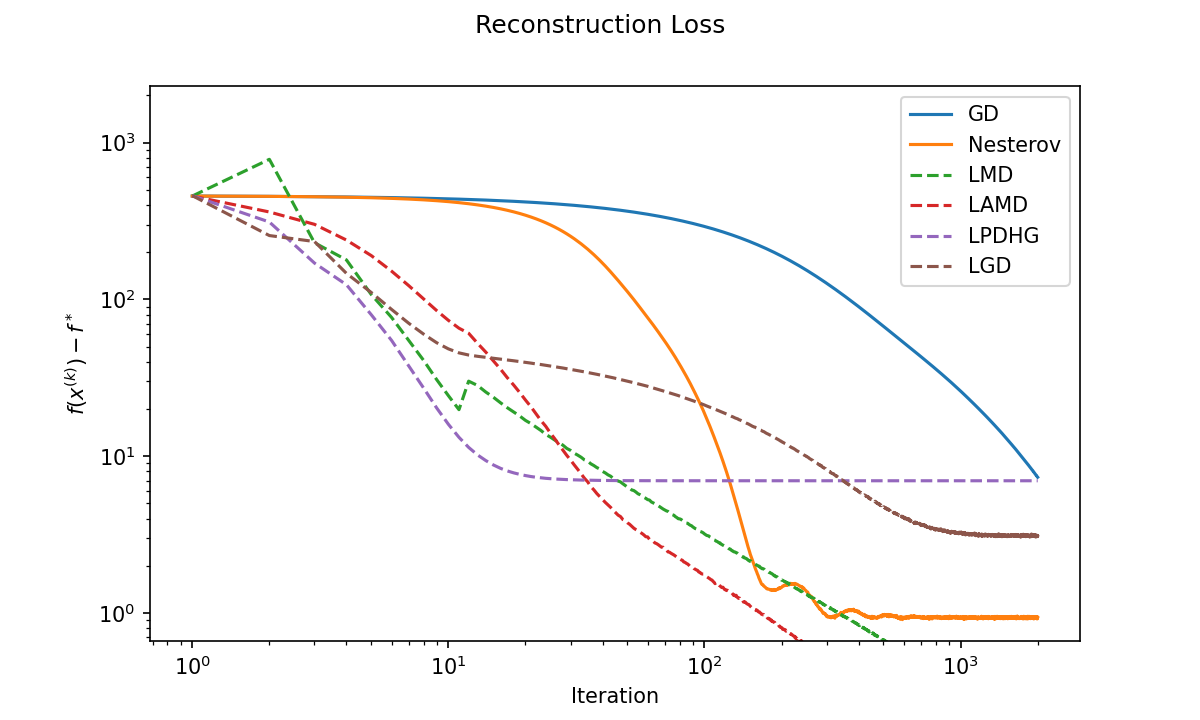}
    \caption{Example convergence profiles for the learned mirror descent algorithm (LMD), the learned accelerated mirror descent algorithm (LAMD), the learned PDHG method (LPDHG), and gradient descent with learned step-sizes (LGD). The target optimization problem arises from TV model-based denoising. We observe that all learned methods are significantly faster than the corresponding gradient descent and Nesterov accelerated gradient descent algorithms in the low iteration regime. However, without special parameter choices for LPDHG, it does not converge to a minimizer.}
    \label{fig:LMD_LAMD}
\end{figure}


\subsection{Plug-and-play methods and data-driven regularization}
\label{sec:Reg_byPnP}
Denoising is the simplest and arguably the most well-studied inverse problem in imaging, with numerous algorithms being developed over the past few decades, particularly for removing additive white Gaussian noise. A natural question is whether one can leverage off-the-shelf denoisers for solving more complicated image recovery tasks with a non-trivial forward operator (that is different from the identity). Venkatakrishnan et al. \cite{venkat_pnp_6737048} pioneered the idea of using denoisers within proximal splitting algorithms (such as ADMM) in a plug-and-play (PnP) fashion, resulting in a class of algorithms known as the PnP denoising approach. To motivate replacing proximal operators with denoisers, let us recall the definition of the proximal operator with respect to a (potentially non-smooth) convex functional $g:\mathbb{X} \rightarrow \mathbb{R}\cup \{+\infty\}$ and a step-size $\tau>0$:   
\begin{align}
	\text{prox}_{\tau g}(x) & = \argmin_u \, \frac{1}{2}\|x - u\|^2 + \tau g(u). 
	\label{eq:prox_1}
\end{align}
As indicated by \eqref{eq:prox_1}, evaluating the proximal operator amounts to denoising a noisy image $x$ using the Bayesian \textit{maximum a-posteriori probability} (MAP) estimation framework with a Gibbs prior $\propto \exp\left(-\tau g(u)\right)$. This denoising interpretation of proximal operators underlies the foundation of PnP approaches, which have been shown to produce excellent reconstruction results for a wide range of imaging inverse problems. A classic and widely popular example of PnP denoising would be to consider it in conjunction with forward-backward splitting (FBS), leading to the following iterative reconstruction algorithm:
\begin{equation}
    x_{k+1} = D_{\sigma} \left(x_k - \eta_k \nabla f(x_k) \right). \label{eq:pnp_fbs}
\end{equation}
Here, $f$ denotes the data fidelity loss for the underlying inverse problem, $\eta_k>0$ is the step-size at iteration $k$, and $D_{\sigma}$ is a denoiser that eliminates Gaussian noise of standard deviation $\sigma$ from its input. 

\noindent Besides the PnP denoising framework within proximal methods, wherein a denoiser implicitly acts as a regularizer, Romano et al. \cite{red_romano} proposed an alternative approach to explicitly construct a regularizing term from a denoiser $D_{\sigma}(x)$ as
\begin{align}
	g(x)=\frac{1}{2}x^\top \left(x-D_{\sigma}(x)\right). 
	\label{eq:red_construction}                         \end{align}
One can then seek to minimize the energy functional $f(x)+\lambda\,g(x)$, where $g$ is as defined in \eqref{eq:red_construction}, leading to fixed-point iterative schemes known as the regularization-by-denoising (RED) algorithms. Nevertheless, it was shown subsequently by Schniter et al. \cite{red_schniter} that the \textit{energy minimization} interpretation of the RED algorithms is valid only when (i) the denoiser is \textit{locally homogeneous}, i.e., $D_{\sigma}\left((1+\epsilon)x\right)=(1+\epsilon)D_{\sigma}(x)$ holds for all $x$ with sufficiently small $\epsilon$, and (ii) the Jacobian of $D_{\sigma}$ is symmetric. These conditions are generally not satisfied by generic denoisers, thereby invalidating the energy minimization-based interpretation of RED. Instead, the authors of \cite{red_schniter} developed a new framework called \textit{score-matching} to analyze the convergence of RED algorithms. 



Notwithstanding their empirical success, PnP denoising algorithms such as \eqref{eq:pnp_fbs} do not immediately inherit the convergence properties of the corresponding optimization scheme, such as FBS in the previous example. Studying the convergence of PnP denoising has received a significant amount of attention in the mathematical imaging community in recent years. Arguably, the most natural form of convergence for PnP algorithms of the form \eqref{eq:pnp_fbs} is the stability of the iterations, ascertaining whether the sequence of iterates $x_k$ generated by a PnP algorithm converges. Such convergence guarantees are typically derived from fixed point theorems, which require showing that the PnP iterations are contractive maps \cite{pnp_admm_chan_2017,pmlr-v97-ryu19a}. For instance, \cite{pmlr-v97-ryu19a} established the fixed-point convergence of PnP-ADMM (i.e., PnP with the \textit{alternating direction method of multipliers} algorithm) under the assumption of Lipschitz continuity of the operator $\left(D_{\sigma}-\id\right)$. The specific result is stated in Theorem \ref{thm:pnp_admm_fp}. 
\begin{theorem}[Fixed-point convergence of PnP-ADMM \cite{pmlr-v97-ryu19a}]
	\label{thm:pnp_admm_fp}
	Consider the PnP-ADMM algorithm, given by
	\begin{align}
		x_{k+\frac{1}{2}} & =\prox_{\tau f}\left(z_k\right),\quad x_{k+1} = D_{\sigma}\left(2x_{k+\frac{1}{2}}-z_k\right), \text{\,\,and\,\,}\nonumber \\ z_{k+1}&=z_k+x_{k+1}-x_{k+\frac{1}{2}},
		\label{pnp_drs1}
	\end{align}
	where the data-fidelity loss $f$ is assumed to be $\mu$-strongly convex. One can equivalently express \eqref{pnp_drs1} as the fixed-point iteration $z_{k+1}=\Op{T}(z_k)$, where
	\begin{eqnarray}
		\Op{T}=\frac{1}{2}\id + \frac{1}{2}\left(2D_{\sigma}-\id\right)\left(2\,\prox_{\tau f}-\id\right).
		\label{pnp_drs_fp}
	\end{eqnarray}
	Suppose that the denoiser $D_\sigma$ satisfies
	\begin{equation}
		\left\|\left(D_{\sigma}-\id\right)(u)-\left(D_{\sigma}-\id\right)(v)\right\|_2 \leq \epsilon  \left\|u-v\right\|_2, 
		\label{cond_denoiser_pnpDRS}
	\end{equation}
	for all $u,v\in \mathbb{X}$ and some $\epsilon>0$, and the strong convexity parameter $\mu$ is such that $\displaystyle\frac{\epsilon}{(1+\epsilon-2\epsilon^2)\,\mu}<\tau$ holds. Then the operator $\Op{T}$ is contractive and the PnP-ADMM algorithm is fixed-point convergent. That is, $\left(x_k,z_k\right)\rightarrow (x_{\infty},z_{\infty})$, where $(x_{\infty},z_{\infty})$ satisfy 
	\begin{eqnarray}
		x_{\infty}=\prox_{\tau  f}\left(z_{\infty}\right) \text{\,\,and\,\,} x_{\infty} = D_{\sigma}\left(2x_{\infty}-z_{\infty}\right).
		\label{pnp_drs_final}
	\end{eqnarray}
	As noted in \cite{pmlr-v97-ryu19a}, fixed-point convergence of PnP-ADMM follows from monotone operator theory if $\left(2D_{\sigma}-\id\right)$ is non-expansive, but \eqref{cond_denoiser_pnpDRS} imposes a less restrictive condition on the denoiser.
\end{theorem}
While fixed-point convergence ensures that the PnP iterations are stable, the specific fixed point to which they converge does not naturally minimize a variational energy function. To bridge the gap between classical variational approaches and PnP methods, it is important to derive conditions under which the limit point of PnP iterations can be characterized as the minimizer (or, at least a stationary point) of some regularized variational objective (which, of course, depends on the denoiser). This type of convergence is referred to as \textit{objective convergence} and is stronger than fixed-point convergence. 

Objective convergence of PnP with classical (pseudo-) linear denoisers (e.g., non-local means denoiser) has been established in \cite{nair2021fixed}. Hurault et al. \cite{gs_denoiser_hurault_2021} showed that PnP with a denoiser constructed as a gradient field, referred to as gradient-step (GS) denoisers, converges to the stationary point of a (possibly non-convex) variational objective (c.f. Theorem \ref{thm:pnp_admm_gsd}). The construction of GS denoisers is motivated by Tweedie's identity; the optimal minimum mean-squared error (MMSE) Gaussian denoiser is given by
\begin{equation}
	D_{\sigma}^*(x):=\mathbb{E}\left[\mathbf{x}_0|\mathbf{x}=x\right] = x+\sigma^2\,\nabla \log p_{\sigma}(x).
	\label{eq:tweedie}
\end{equation}
Here, $\mathbf{x}=\mathbf{x}_0+\sigma\,\mathbf{w}$, where $\mathbf{w}\sim\mathcal{N}(0,I)$, is the Gaussian-noise corrupted version of the clean image $\mathbf{x}_0\in \mathbb{R}^d$ and 
\begin{equation}
	p_{\sigma}(x)=\frac{1}{(2\pi\sigma^2)^{\frac{d}{2}}}\int\exp\left(-\frac{\|x-x_0\|_2^2}{2\sigma^2}\right)p(x_0)\,\mathrm{d}x_0.
	\label{eq:smoothed_pdf}
\end{equation}
Indeed, the optimal Gaussian denoiser is of the form $D_{\sigma}^*(x)=x-\nabla\,g^{*}_{\sigma}(x)$, where $g^{*}_{\sigma}$ is the negative log-density of the smoothed distribution $p_{\sigma}$ defined in \eqref{eq:smoothed_pdf}. This has a structure identical to that of a GS denoiser, parameterized as $D_{\sigma}(x)=x-\nabla\,g_{\sigma}(x)$. It was argued in \cite{gs_denoiser_hurault_2021} that directly parameterizing $g_{\sigma}$ using a deep neural network does not lead to state-of-the-art denoising performance, but instead, modeling $g_{\sigma}$ as $g_{\sigma}(x)=\frac{1}{2}\left\|x-N_{\sigma}(x)\right\|_2^2$ for a differentiable network $N_{\sigma}(x)$ produces superior denoising performance. The denoiser is trained by minimizing the MSE, given by $J:=\mathbb{E}_{\mathbf{x},\mathbf{w}}\left\|D_{\sigma}(\mathbf{x}+\sigma\,\mathbf{w})-\mathbf{x}\right\|_2^2$, where $\mathbf{w}\sim\mathcal{N}(0,I)$, approximated over the training dataset consisting of the ground-truth images and their noisy counterparts.

\begin{theorem}[Objective convergence of PnP iterations \cite{gs_denoiser_hurault_2021}]
	\label{thm:pnp_admm_gsd}
	Suppose the denoiser is a gradient-step (GS) denoiser $D_{\sigma}=\id-\nabla g_{\sigma}$, where $g_{\sigma}$ is proper, lower semi-continuous, and differentiable with $L$-Lipschitz gradient. The GS-PnP algorithm proposed in \cite{gs_denoiser_hurault_2021} is given by
	\begin{align}
		x_{k+1} & = \prox_{\tau f}\left(x_k-\tau \lambda \nabla g_{\sigma}(x_k)\right)\nonumber \\ &=\prox_{\tau  f}\circ \left(\tau\lambda\, D_{\sigma}+(1-\tau\lambda\, \id)\right)(x_k),
		\label{eq:pnp_gs_hqs}
	\end{align}
	where $f \colon \mathbb{X}\to\Real\cup \{+\infty\}$ is a convex and lower semi-continuous data-fidelity term. Then, the following guarantees hold for $\tau<\frac{1}{\lambda\, L}$:
	\begin{enumerate}
		\item The sequence $F(x_k)$, where $F=f+\lambda\, g_{\sigma}$, is non-increasing and convergent.
		\item The residual $\left\|x_{k+1}-x_k\right\|_2$ converges to 0. 
		\item All limit points of $\{x_k\}$ are stationary points of $F(x)$. 
	\end{enumerate}
	Notably, the PnP iteration defined by \eqref{eq:pnp_gs_hqs} is exactly equivalent to proximal gradient descent on $f+\lambda\,g_{\sigma}$, with a potentially non-convex $g_{\sigma}$.
\end{theorem}

\begin{figure}[t!]
    \centering
    \subfloat[\centering PnP-LBFGS]{{\includegraphics[width=0.33\textwidth]{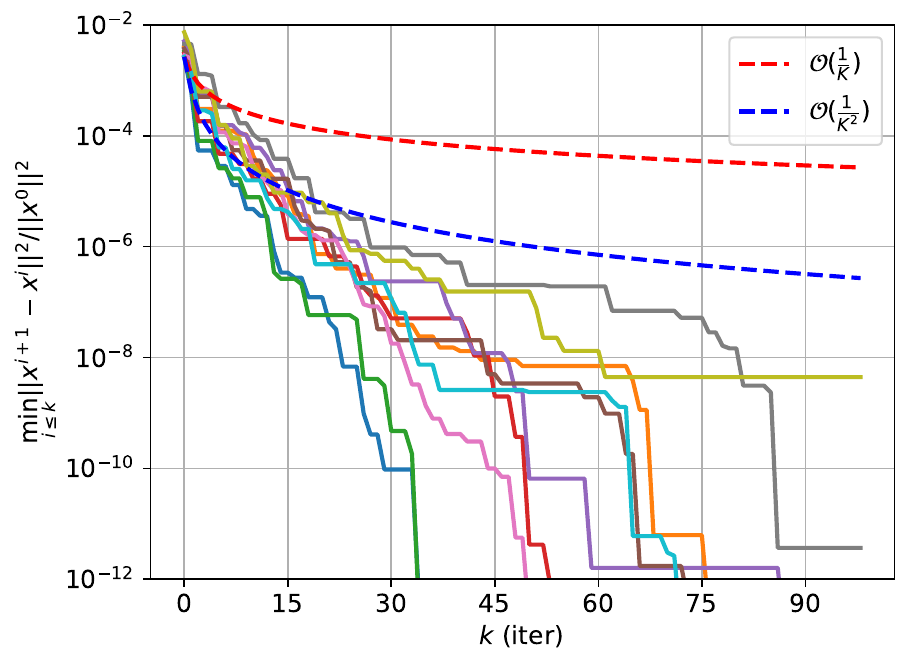}}}%
    \subfloat[\centering PnP-FISTA]{{\includegraphics[width=0.33\textwidth]{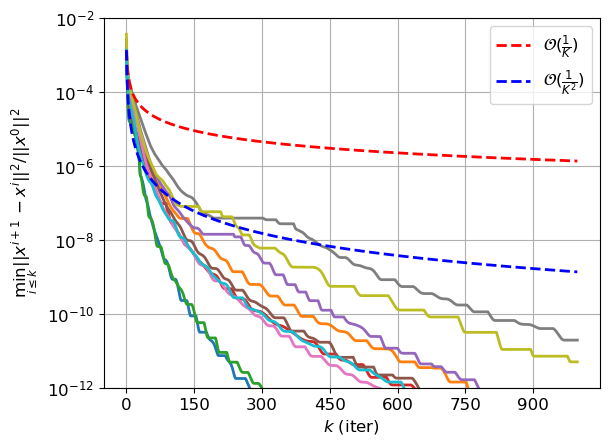}}}%
    \subfloat[\centering PnP-$\alpha$PGD]{{\includegraphics[width=0.33\textwidth]{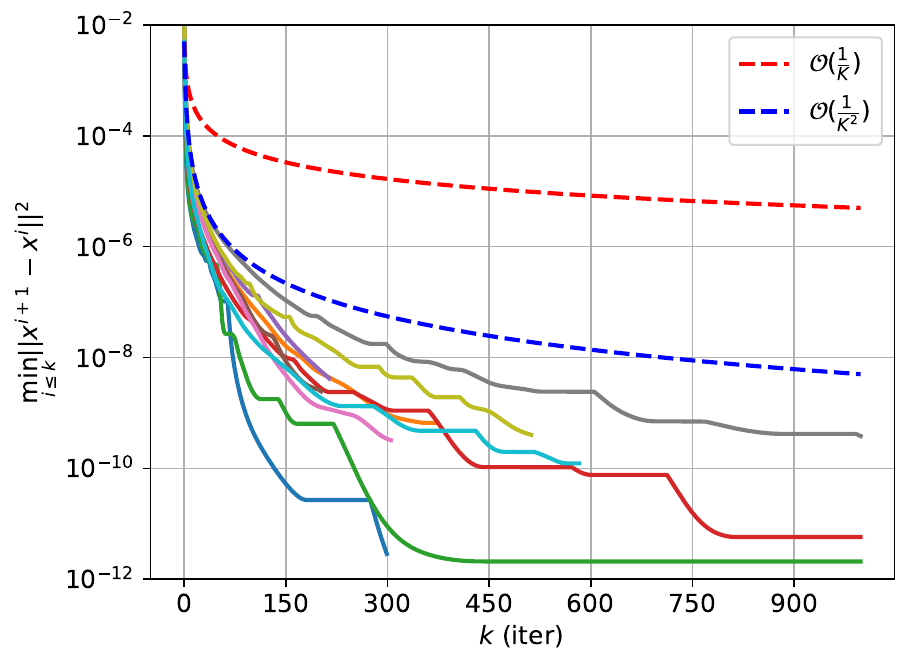}}}

    \subfloat[\centering PnP-PGD]{{\includegraphics[width=0.33\textwidth]{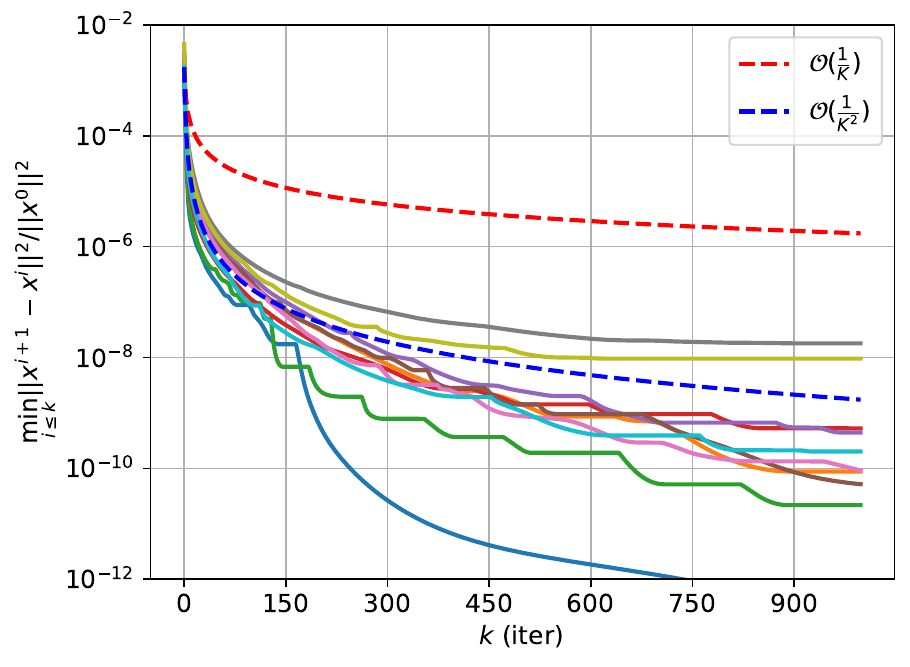}}}%
    \subfloat[\centering PnP-DRS]{{\includegraphics[width=0.33\textwidth]{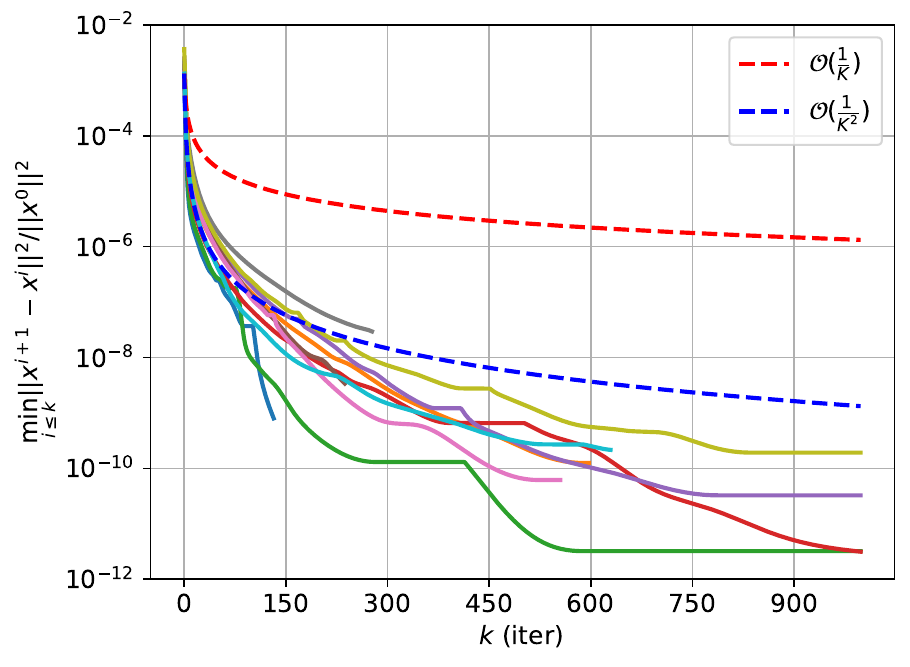}}}%
    \subfloat[\centering PnP-DRSDiff]{{\includegraphics[width=0.33\textwidth]{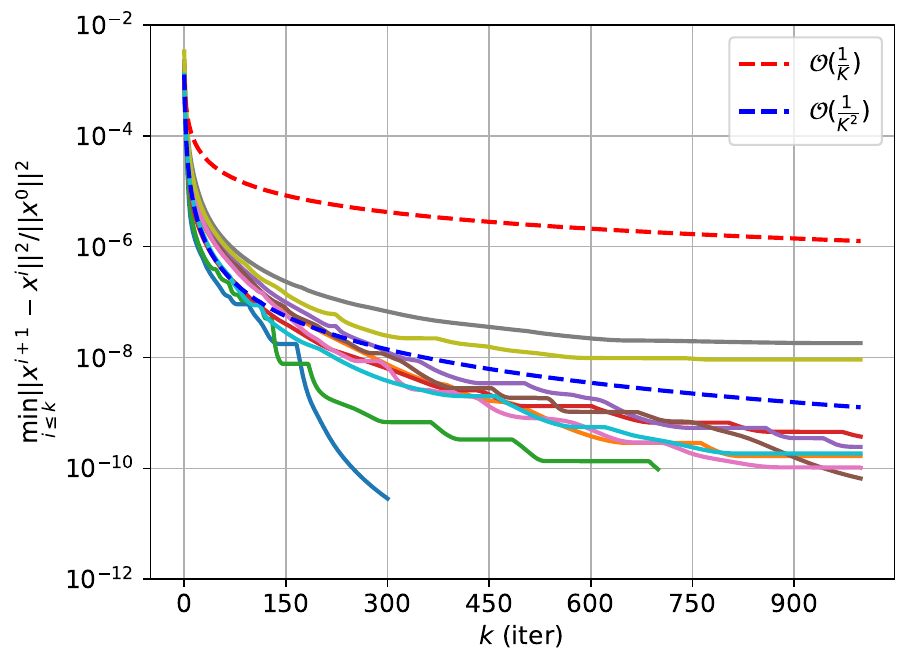}}}%
    
    \caption{Convergence of the residuals $\min_{i \le k} \|x^{i+1} - x^{i}\|^2/\|x^0\|^2$ for various PnP methods applied to image deblurring. Each curve corresponds to one image from the CBSD10 dataset, corrupted with 3\% additive Gaussian noise. Except for PnP-FISTA, stationary points of these PnP methods are critical points of a weakly convex function, corresponding to the noisy image and the denoiser.}
    \label{fig:PnP_residuals}
\end{figure}

While objective convergence ensures a one-to-one connection between PnP iterates with the minimization of a variational objective, it does not provide any guarantees about the regularizing properties of the solution that the iterates converge to. In the same spirit as classical regularization theory, it is therefore desirable to be able to control the implicit regularization effected by the denoiser in PnP algorithms and analyze the asymptotic behavior of the PnP reconstruction as the noise level and the regularization strength tend to vanish. More precisely, assuming that the PnP iterations converge to a solution $\hat{x}\left(y^\delta,\sigma,\lambda\right)$, where $\sigma$ is a parameter associated with the denoiser and $\lambda$ is an explicit regularization penalty, one would like to obtain appropriate selection rules for $\sigma$ and/or $\lambda$ such that $\hat{x}\left(y^\delta,\sigma,\lambda\right)$ is a convergent regularization scheme in the limit as $\delta\rightarrow 0$. To the best of our knowledge, some progress in this direction was first made in \cite{ebner2022plugandplay}, and the precise convergence result is stated in Theorem \ref{thm:pnp_conv_reg_haltmeyer}. A similar convergence result for PnP methods in the sense of regularization was shown in \cite{hauptmann2023convergent} considering linear denoisers, together with a systematic approach based on \textit{spectral filtering} for controlling the regularization effect arising from such denoisers.  
\begin{theorem}[Convergent plug-and-play (PnP) regularization \cite{ebner2022plugandplay}]
	\label{thm:pnp_conv_reg_haltmeyer}
	Consider the PnP-FBS iterates of the form
	\begin{equation}
		x_{\lambda,k+1}^{\delta} = D_{\lambda}\left(x_{\lambda,k}^{\delta}-\eta\,A^*\left(Ax_{\lambda,k}^{\delta}-y^{\delta}\right)\right),
		\label{eq:pnp_conv_reg_thm}
	\end{equation}
	where $D_{\lambda}$ is a denoiser with a tuneable regularization parameter $\lambda$. Let $\PnP\left(\lambda,y^{\delta}\right)$ be the fixed point of the PnP iteration \eqref{eq:pnp_conv_reg_thm}. For any $y\in\range(A)$ and any sequence $\delta_k>0$ of noise levels converging to $0$, there exists a sequence $\lambda_k$ of regularization parameters converging to $0$ such that for all $y_k$ with $\|y_k-y^0\|_2\leq \delta_k$, the following hold under appropriate assumptions on the denoiser (see Definition 3.1 in \cite{ebner2022plugandplay} for details):
	\begin{enumerate}
		\item $\PnP\left(\lambda,y^{\delta}\right)$ is continuous in $y^{\delta}$ for any $\lambda>0$;
		\item The sequence $\left(\PnP\left(\lambda_k,y_k\right)\right)_{k\in\mathbb{N}}$ has a weakly convergent subsequence; and
		\item The limit of every weakly convergent subsequence of $\left(\PnP\left(\lambda_k,y_k\right)\right)_{k\in\mathbb{N}}$ is a solution of the operator equation $y^0=Ax$.
	\end{enumerate}
\end{theorem}

    \begin{figure}[h!]%
    \centering
    \subfloat[\centering Ground-Truth ]{{\includegraphics[height=4cm]{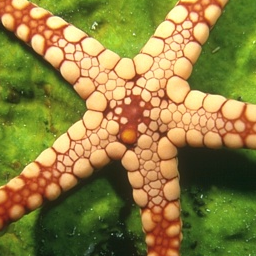}} } %
    \subfloat[\centering PnP-LBFGS(29.78dB)]{{\includegraphics[height=4cm]{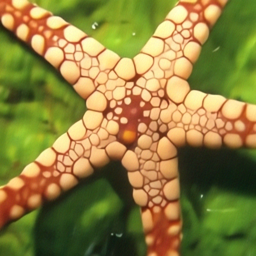}} } %
    \subfloat[\centering PnP-PGD (28.68dB)]{{\includegraphics[height=4cm]{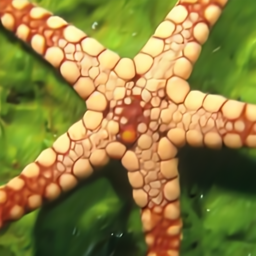}} } %
    
    \subfloat[\centering Corrupted ]{{\includegraphics[height=4cm]{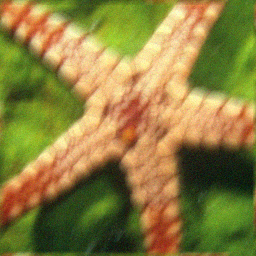}} } %
    \subfloat[\centering PnP-DRSdiff (28.66dB)]{{\includegraphics[height=4cm]{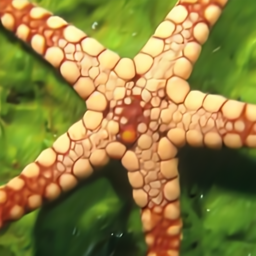}} } %
    \subfloat[\centering PnP-DRS (29.39dB)]{{\includegraphics[height=4cm]{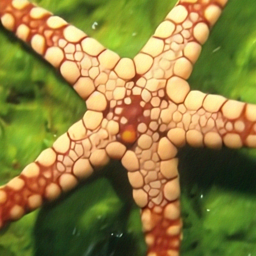}} } %
     \caption{Deblurring visualization using starfish image, with each method limited to a maximum of 100 iterations. Experiments are run with additive Gaussian noise $\sigma= 7.65$. PnP-LBFGS converges within the first 100 iterations, while the other PnP algorithms take longer to converge.}
\end{figure}

In Figure 13 and 14, we present some numerical results from \cite{tan2023provably} on applying provably convergent PnP algorithms including PnP-LBFGS, PnP-PGD, PnP-DRS etc, on image deblurring task for illustration, more details can be found in the referenced paper.

\section{Various ground-truth-free approaches for image reconstruction}\label{ssec:GTfreeIR}
In this section, we briefly survey some other closely related unsupervised training strategies for imaging inverse problems. The frameworks described here are mostly suitable for the cases where we have limited training data for the networks, for example, in medical tomographic imaging we could have plenty of noisy sinogram measurement data from the imaging devices, but a very limited amount of data for ground-truth images. Strictly speaking, there are sometimes no actual ``ground-truths'' in practice, making the use of unsupervised schemes necessary.
\subsection{Deep image prior}
One of the popular and empirically successful unsupervised approaches for imaging is the deep image prior (DIP) method \cite{ulyanov2018deep}. Surprisingly, this approach requires no training data, relying completely on the regularization effect of the architecture of the deep CNNs and implicit regularization of the gradient-based optimizers \cite{tachella2021neural}. Let us denote a neural network such as a U-net by $G_\theta: \R^{d'} \rightarrow \R^d$, which can be either untrained or pretrained, parameterizing the image to be reconstructed. For an arbitrary vector $z$, the DIP scheme can be written as minimizing \emph{approximately}:
\begin{equation}
    \theta^\star \approx \argmin_\theta \|y - AG_\theta(z)\|_2^2,
\end{equation}
with some first-order methods such as Adam, with early-stopping to avoid overfitting. The final reconstruction is then computed as $x^{\star}=G_{\theta^\star}(z)$. While letting $z$ be chosen as a Gaussian random vector produces reasonable results, it has been observed that warm-starting by choosing $z$ to be the corrupted image input leads to better results. For example, when applying DIP in denoising, it is better to choose $z$ to be the noisy input image itself for faster convergence and improved results, as observed by Tachella et al \cite{tachella2021neural}. This work also demonstrates that the success of DIP is due to the implicit regularization by the network architecture and the dynamics of the gradient-based optimizer.

\noindent Despite the nonstandard reconstruction method, the DIP approach demonstrates remarkable numerical performance without any training data, even in highly ill-posed inverse problems such as inpainting with many missing pixels. Although this scheme is usually numerically inferior compared to fully-supervised schemes, the DIP approach demonstrates that the implicit regularization jointly formed by the architecture and gradient-based optimization is already a very strong regularization for imaging. Moreover, it can be jointly applied with classical variational regularization methods and plug-and-play priors introduced in the previous subsections for even better reconstruction results. For example, the DIP-TV approach \cite{liu2019image}
\begin{equation}
    \theta^\star \approx \argmin_\theta \|y - AG_\theta(z)\|_2^2 + \mu\|\nabla G_\theta(z)\|_1,
\end{equation}
and the DIP-RED approach \cite{mataev2019deepred}, given by
\begin{equation}
    \theta^\star \approx \argmin_\theta \|y - AG_\theta(z)\|_2^2 + \mu \,G_\theta(z)^{\top}(G_\theta(z) - D_\lambda(G_\theta(z))),
\end{equation}
both fall within the category of combining DIP with additional prior terms. With the assistance of additional regularization, the performance of DIP is often improved, and the need for early stopping is alleviated if the regularization parameter $\mu$ is appropriately chosen.

\subsection{Noise-2-X methods}


The Noise2Noise scheme takes two distinct noisy observations of natural images for training denoisers without ground-truth image, by taking one of the noisy observations as a ``ground-truth'' in the fidelity term \cite{lehtinen2018noise2noise}. An interesting class of similar ground-truth-free unsupervised training schemes has been developed, such as Noise2Self \cite{batson2019noise2self}, Noise2Void \cite{krull2019noise2void}, Noisier2Noise \cite{moran2020noisier2noise}, and many other related schemes \cite{quan2020self2self, ke2021unsupervised, huang2021neighbor2neighbor}. We refer to this class of training schemes as the \textit{Noise-2-X} methods. Given a collection of noisy/corrupted images $\{\hat{x}_i\}_{i=1}^n$ and a neural network to train, typically deep CNNs or U-nets, the Noise-2-X schemes train the reconstruction network on unsupervised losses of the form:
\begin{equation}\label{n2n}
    \theta^\star \approx \argmin_\theta \frac{1}{n}\sum_{i=1}^n\|\hat{y}_i - G_\theta(\hat{x}_i)\|_2^2,
\end{equation}
where $\{\hat{x}_i, \hat{y}_i\}_{i=1}^n$ are pairs of noisy perturbations of the inaccessible ground-truth images $\{x_i\}_{i=1}^n$. Different noise-2-X schemes consider different choices of such perturbations. The aim of using pairs of perturbations is to use \eqref{n2n} to approximate the supervised loss
\begin{equation}\label{n2n_sup}
    \theta^\star \approx \argmin_\theta \frac{1}{n}\sum_{i=1}^n\|x_i - G_\theta(\hat{x}_i)\|_2^2,
\end{equation}
in the absence of ground-truth images $x_i$. For example, consider the denoising problem $y_i = x_i + \varepsilon_i$ where $\varepsilon_i$ denotes additive Gaussian noise. The unsupervised loss can be written as $\frac{1}{n}\sum_{i=1}^n\|y_i - G_\theta(\hat{x}_i)\|_2^2 = \frac{1}{n}\sum_{i=1}^n\|x_i + \varepsilon_i - G_\theta(\hat{x}_i)\|_2^2$. The gradient of this approximation is an unbiased estimate of the gradient for the supervised loss above, and such an approximation becomes increasingly accurate as the sample size $n$ increases. Similar to the DIP, denoising networks based on noise-2-X schemes are also trained using gradient-based optimization algorithms such as Adam or SGD.

In imaging tasks such as natural image denoising, these unsupervised training schemes demonstrate reasonably good performance, closely matching the performance of denoising networks with fully supervised training. Combined with the plug-and-play schemes we have introduced before, the denoisers trained by these noise-2-X schemes can be also applied to solve more sophisticated imaging inverse problems such as deblurring, inpainting, and tomographic reconstruction in the absence of any noise-free ground-truth images.

\subsection{Equivariant imaging}
In certain imaging applications such as CT or MRI reconstruction, we often only have low-quality measurements $\{y_i\}_{i=1}^n$ without any ground-truth images. This situation restricts the use of supervised training, where synthetic data is instead used. In such cases, the quality of the measurements significantly affects the training quality of brute-force unsupervised training:
\begin{equation}
    \theta^\star \approx \argmin_\theta \frac{1}{n}\sum_{i=1}^n\|y_i - AG_\theta(y_i)\|_2^2.
\end{equation}
This unsatisfactory training is due to the difficulty of learning in the presence of highly non-trivial null-spaces. To mitigate this, Chen et al. \cite{chen2021equivariant} proposed the Equivariant Imaging (EI) framework, utilizing the equivariant structure of the forward operator to improve the performance of the unsupervised training in this context. More precisely, in the majority of imaging inverse problems, the plausible set of images $\mathcal{I}$ are invariant to a certain group of transformations $\mathcal{G} = \{g_1, g_2,..., g_{|\mathcal{G}|}\}$ with actions $T_g$ such that $T_gx \in \mathcal{I}$ for all $x \in \mathcal{I}$. For example, natural images are usually invariant to shift operations, while CT/MRI images are usually invariant to rotations. Exploiting this structure of the plausible image set, the desired neural network solution should approximately satisfy:
\begin{equation}
    G_\theta(AT_gx) = T_gG_\theta(Ax). 
\end{equation}
The composite map $h_\theta \circ A$ should be equivariant under the transformations $T_g$, meaning that the operators commute. This leads to the EI training framework:
\begin{multline}\label{ei}
    \theta^\star \approx \argmin_\theta \frac{1}{n}\sum_{i=1}^n\|y_i - AG_\theta(y_i)\|_2^2  + \mu \Expect_{g \in\mathcal{ G}}\left[\|G_\theta(AT_gG_\theta(y_i)) - T_gG_\theta(AG_\theta(y_i))\|_2^2\right].
\end{multline}
This is the unsupervised training loss with the addition of a regularization term that encourages the network to utilize the equivariant structure of the imaging problem. Akin to the previously introduced unsupervised methods, gradient-based optimization solvers such as Adam are applied for training, with an extra computational overhead due to the sophisticated regularization term. Although training using EI is more computationally expensive and requires much more memory compared to a brute-force approach, this framework demonstrates remarkable numerical potential and can match the accuracy of fully supervised approaches closely \cite{chen2021equivariant}. The EI framework can enable practitioners to train advanced reconstruction networks such as FBP-ConvNet and deep unrolling networks from only the measurement data without the ground-truth images.

\subsection{Stein's unbiased risk estimation (SURE)}
An unsupervised learning approach based on Stein's unbiased risk estimation (SURE) \cite{sure_stein} was proposed by Metzler et al. \cite{metzler2020unsupervised}. The estimation problem considered in \cite{metzler2020unsupervised} was that of recovering an image $x\in\mathbb{R}^n$ from its linearly degraded measurement $y=Ax+\mathbf{w}$, where $\mathbf{w}$ is Gaussian with mean zero and covariance $\sigma_{\mathbf{w}}^2I$. Then, it can be shown that 
\begin{equation}
J(\theta) := \mathbb{E}_{\mathbf{w}}\left[\frac{1}{n}\|y-G_{\theta}(y)\|_2^2\right]-\sigma_{\mathbf{w}}^2+\frac{2\sigma_{\mathbf{w}}^2}{n}\text{div}_y\left(G_{\theta}(y)\right),
    \label{eq:sure}
\end{equation}
where $\text{div}$ denotes the divergence operator, is an unbiased estimator of the mean-squared error (MSE) $\displaystyle \mathbb{E}_{\mathbf{w}}\left[\frac{1}{n}\|x-G_{\theta}(y)\|_2^2\right]$. Since approximating $J(\theta)$ requires only the measured data and not the corresponding ground-truth images, it serves as a surrogate loss for MSE and results in an unsupervised learning framework. To approximate the divergence term, the authors of \cite{metzler2020unsupervised} adopted a Monte Carlo-based approach that relies on the following:   
\begin{equation}\label{eq:div}
    \text{div}_y\left(G_{\theta}(y)\right)= \underset{\epsilon\rightarrow 0}{\lim}\,\mathbb{E}_{\mathbf{u}}\left[\mathbf{u}^\top\left(\frac{G_{\theta}(y+\epsilon\mathbf{u})-G_{\theta}(y)}{\epsilon}\right)\right],
\end{equation}
where $\mathbf{u}\sim\mathcal{N}(0,I)$. A similar unbiased estimator of the MSE can be derived for noise distributions in the exponential family. SURE can be utilized as a general framework that can turn any generic supervised MSE-based learning approach (for instance, a bilevel learning framework) into an unsupervised one by replacing the MSE with its SURE-based estimate. 
\subsubsection{Robust equivariant imaging via SURE}
The EI unsupervised training framework introduced in the previous subsection can also be further improved in terms of robustness to measurement noise by incorporating SURE, as shown in the work of Chen et al \cite{chen2022robust}. There is a major weakness of the EI approach regarding the fragility towards measurement noise, such that as the measurement noise increases, the performance of EI would experience very significant decay. An effective remedy for this issue turns out to be utilizing the SURE loss \eqref{eq:sure}. This modified robust EI framework can be summarized as the following objective:
\begin{align*}
    \theta^\star \approx \argmin_\theta &\frac{1}{n}\sum_{i=1}^n\|y_i - AG_\theta(y_i)\|_2^2+2\sigma^2\text{div}_{y_i}\left(G_{\theta}(y_i)\right)\\& +  \mu\, \Expect_{g \in\mathcal{ G}}[\|G_\theta(AT_gG_\theta(y_i)) - T_gG_\theta(AG_\theta(y_i))\|_2^2].     \end{align*}
According to \eqref{eq:div}, when training the reconstruction networks using gradient-based methods, the divergence term can be simply approximated by:
\begin{equation}\label{eq:div_approx}
    \text{div}_y\left(G_{\theta}(y)\right)\approx \mathbf{u}^\top\left(\frac{G_{\theta}(y+\varepsilon\mathbf{u})-G_{\theta}(y)}{\varepsilon}\right),
\end{equation}
in each iteration, with $\mathbf{u} \sim \mathcal{N}(0,I)$ while $\varepsilon$ being chosen to be a small constant. With this modified loss, the resulting reconstruction networks can closely match fully supervised methods even when the noise in the measurement is significant.

\section{Summary and conclusions}
Unsupervised learning is a powerful method of performing machine learning in the absence of complete ground-truth data, such as unpaired training examples, and access to samples of only the ground-truth images, or of only noisy measurements. We presented three paradigms, namely probabilistic approaches based on optimal transport and cycle architectures, learned priors through learning-to-optimize and plug-and-play, as well as various ways of inserting prior knowledge for regularization. Each of these paradigms requires some prior knowledge, such as a degradation model or probabilistic interpretation. Nonetheless, such models have been shown to be competitive with supervised models, and are applicable to more general classes of problems.

In \Cref{sec:OTapproaches}, we reviewed unsupervised approaches based on optimal transport, particularly the Cycle-WGAN approach consisting of two WGANs in opposite directions and the adversarial regularization method where a regularizer is parameterized using a neural network and learned adversarially. Both approaches aim to minimize a Wasserstein distance between distributions induced by the learned components, and an approach combining both Cycle-WGAN approaches and adversarial regularization was discussed in \Cref{sssec:combiningE2EAR}. These approaches have the benefit of having a probabilistic interpretation, where the distribution of the generated or reconstructed data lives in a certain neighborhood of the ground-truth distribution. This lies in the intersection of learning the prior and posterior distributions, and can also be related to semi-supervised learning, where there is an imbalance of measurements and ground-truths.

Several convex analysis-based methods for unsupervised learning were presented in \Cref{sec:unsupervised}. In particular, the learning-to-optimize, which accelerates model-based reconstruction, was considered in \Cref{ssec:l2o}. Plug-and-play methods for image reconstruction tasks, where an image prior is implicitly defined by a pre-trained Gaussian denoiser were considered in \Cref{sec:Reg_byPnP}. \Cref{ssec:GTfreeIR} detailed several training methods for one-shot image reconstruction such as using the deep image prior, or various methods for training denoisers in the absence of ground-truth data.

In this review, we focused on works that derive from classical results in optimal transport and convex analysis. However, the scope of unsupervised learning is much broader once this restriction is lifted. Notable examples include physics-informed neural networks, which aim to learn physical operators such as PDEs or dynamical systems \cite{karniadakis2021physics}. While a lot of theory already exists for unsupervised learning, we believe that the following few issues are particularly important for closing the gap between unsupervised and supervised methods:

\begin{enumerate}[leftmargin=*]
    \item There is an inherent difference in information available in the supervised regime compared to the unsupervised regime. Some works already seek to quantify this, such as \cite{tachella2023sensing} that rephrases the EI framework in terms of compressed sensing, and derives bounds for signal recovery based on classical theorems. An interesting direction would be quantifying the performance difference induced by this information gap, as well as in suitable limiting cases. 

    \item Unsupervised methods were categorized into three main classes as in \Cref{sssec:unsupervised}, and all these formulations assume some sort of prior information into the model. The works presented in this review are based on classical results in optimal transport and convex analysis, allowing for some theoretical analysis. A more theoretical framework for building unsupervised models, utilizing probabilistic or geometric ideas, could lead to more efficient usage of data and help close the gap between supervised and unsupervised methods. 
\end{enumerate}

\bibliography{references}
\bibliographystyle{plain}


\end{document}